\documentclass[review]{elsarticle}
\usepackage[utf8]{inputenc}
\usepackage{lineno,hyperref,booktabs,tikz,algorithm,algorithmic,pgf,graphicx,url,color,colortbl,bm,amsmath,amssymb,amsfonts,amsthm,array,epsfig,enumerate,bbm,float,mathtools,subcaption,thmtools,environ,tcolorbox,caption,subcaption,pgfplots}
\usepackage{circuitikz}
\usetikzlibrary{trees}
\modulolinenumbers[5]
\tcbuselibrary{skins}
\usepackage[mathscr]{euscript}
\definecolor{g1}{gray}{0.9}
\definecolor{g2}{gray}{0.7}
\tikzstyle{freccia}=[->,color=mygray,line width=1pt]
\xdefinecolor{mygray}{RGB}{0, 0, 0}
\newcommand{\bo}[1]{\textbf{#1}}
\newcommand{\pn}{\mathbb{P}_n}

\newcommand{\interp}[1]{\langle #1 \rangle}
\newcommand{\up}[1]{\overline{#1}}
\newcommand{\lo}[1]{\underline{#1}}

\newtheorem{theorem}{Theorem}
\newtheorem{definition}{Definition}

\newtheorem{example}{Example}

\journal{IJAR}
\begin{document}
\begin{frontmatter}
\title{Tractable Inference in\\Credal Sentential Decision Diagrams}
\author[IDSIA]{Lilith Mattei}
\ead{lilith@idsia.ch}
\author[IDSIA]{Alessandro Antonucci\corref{cor}}
\ead{alessandro@idsia.ch}
\author[USP]{Denis Deratani Mau\'a}
\ead{ddm@ime.usp.br}
\author[IDSIA]{Alessandro Facchini}
\ead{alessandro.facchini@idsia.ch}
\author[USP]{Julissa Villanueva Llerena}
\ead{jgville@ime.usp.br}
\address[IDSIA]{Istituto Dalle Molle di Studi per l'Intelligenza Artificiale, Manno-Lugano, Switzerland}
\address[USP]{Institute of Mathematics and Statistics,
  University of S\~ao Paulo, S\~ao Paulo, Brazil}
\cortext[cor]{Corresponding author.}
\begin{abstract}
\emph{Probabilistic sentential decision diagrams} are logic
circuits where the inputs of disjunctive gates are annotated
by probability values. They allow for a compact
representation of joint probability mass functions defined
over sets of Boolean variables, that are also consistent with
the logical constraints defined by the circuit.
The probabilities in such a model are usually ``learned'' from
a set of observations. This leads to overconfident and
prior-dependent inferences when data are scarce, unreliable or
conflicting.
In this work, we develop the \emph{credal sentential decision
  diagrams}, a generalisation of their probabilistic
counterpart that allows for replacing the local probabilities
with (so-called \emph{credal}) sets of mass functions. These
models induce a joint credal set over the set of Boolean
variables, that sharply assigns probability zero to states
inconsistent with the logical constraints. Three inference
algorithms are derived for these models. These allow to
compute: (i) the lower and upper probabilities of an
observation for  an arbitrary number of variables; (ii) the
lower and upper conditional probabilities for the state of a
single variable given an observation; (iii) whether or not all
the probabilistic sentential decision diagrams compatible with
the credal specification have the same most probable
explanation of a given set of variables given an observation
of the other variables. These inferences are \emph{tractable},
as all the three algorithms, based on bottom-up traversal with
local linear programming tasks on the disjunctive gates, can
be solved in polynomial time with respect to the circuit size.
The first algorithm is always exact, while the remaining two
might induce a conservative (outer) approximation in the case of multiply connected circuits. A semantics for this approximation together with an auxiliary algorithm able to decide whether or not the result is exact is also provided together with a brute-force characterization of the exact inference in these cases. For a first empirical validation, we consider a simple application based on noisy seven-segment display images. The credal models are observed to properly distinguish between easy and hard-to-detect instances and outperform other generative models not able to cope with logical constraints.
\end{abstract}
\end{frontmatter}
\section{Introduction}\label{sec:intro}
Probabilistic graphical models \cite{koller2009,Darwiche2009} are widely
used in machine learning and knowledge-based decision-support
systems, due to their ability to provide compact and
intuitive 
descriptions of joint probability mass functions by exploiting
conditional independence relations encoded in a graph.
However, the
ability to provide compact representation does not imply that
inferences with the model can be computed efficiently
\cite{Roth96,kwisthout2010necessity,DeCampos2011}, and
practicioners need to rely on approximate inference algorithms
with no guarantees.

To allow for fast and accurate inference, some authors have proposed
abandoning the intuitive (declarative) semantics of graphical
models in favor of a more procedural (and less transparent)
representation of probability mass functions as arithmetic (or
logic) circuits \cite{acs,poon2011sum,cutset,kisa2014}. The latter have been broadly termed
\emph{tractable models}, for their ability to provide
polynomial-time inference with respect to the circuit size.
\emph{Sum-product networks} (SPNs) \cite{poon2011sum} are the
most popular example in this area. Remarkably, SPNs can be
also intended as a probabilistic counterpart of deep neural
networks and, when used for machine learning, they offer
competitive performances in many tasks
\cite{peharz2018probabilistic,rat}.

Another prominent example of tractable models are
\emph{probabilistic sentential decision diagrams} (PSDDs)
\cite{kisa2014}. Roughly speaking, a PSDD is a logical circuit
representation of a joint probability mass function that
assigns zero probability to the impossible states of the
underlying logical constraints. Notably, PSDDs allow for
enriching statistical models with knowledge about constraints
in the domain without sacrificing efficient inference
\cite{structured,snb,routes,subset}.

When data are scarce, conflicting or unreliable, learning
sharp estimates of probability values can lead to inferences
that are dominated by the choice of hyperparameters and priors.
The area of \emph{imprecise probabilities} advocate for a more
flexible and robust representation of statistical models, 
through the use of \emph{credal sets}, that is, sets of
probability mass functions induced by a (typically finite)
number of linear constraints \cite{walley1996inferences}.
This lead to the development of generalizations of graphical
models such as \emph{credal networks} \cite{cozman2000}, that
extend Bayesian networks to allow for the representation of
imprecisely specified conditional probability values.

Recently, SPNs have also been extended to the imprecise
probability setting, giving rise to \emph{Credal Sum-Product
  Networks} (CSPNs)  \cite{maua2017credal,
  maua2018robustifying, villanueva2019isipta}.
These models allow for a richer representation of uncertainty
without compromising computational tractability of inferences.

In this work, we develop the \emph{Credal Setential Decision
  Diagrams} (CSDDs), a credal-set extension of probabilistic
sentential decision diagrams that allow for richer
representation of uncertainty with small computational
overhead. Compared to CSPNs, CSDDs allow for a more principled
semantics of local credal sets.

We take advantage of the structural similarities
between PSDDs and SPNs to adapt many of the algorithms
originally proposed for CSPNs
\cite{maua2017credal,villanueva2019isipta} for CSDDs. More
specifically, a PSDD can be seen as a special type of
\emph{selective} SPNs \cite{selective}, where differently from
standard SPNs, \emph{Maximum-A-Posteriori} (MAP) inference and
parameter learning can be performed efficiently
\cite{peharz2016,approx}.
As a result we therefore deliver three algorithms for CSDDs
allowing to compute: (i) \emph{marginals}, that is, the lower
and upper probabilities of an observation of an arbitrary
number of model variables, (ii) \emph{conditionals}, that is,
the lower and upper probabilities of single queried variable
given an observations of some other variables; and (iii)
\emph{MAP robustness}, that is, checking whether or not the
most probable configuration for some queried variables given
an observation of the other ones is the same for all PSDDs
consistent with a CSDD. Those inferences are tractable as all
the algorithms only requires a bottom-up traversal of the
logical circuit underlying the model with local linear
programming tasks to be solved on the disjunctive nodes, thus
being polynomial in the circuit size. The inferences are
always exact for the first task, while for the remaining two
tasks the procedure delivers a conservative (outer) approximation for multiply connected circuits (see Definition \ref{def:topology}). For these cases, a polynomial-time algorithm to check whether or not the inference is exact is also provided together with a bound on the complexity required to compute exact inference by brute 
force.

This paper extends a preliminary version \cite{mattei19}
with the inclusion of the algorithm for MAP robustness, the
characterization of the approximation in the multiply
connected case, and an experimental validation.

The rest of the paper if organized as follows.
In the next section we open the discussion with a toy example to be used along the paper to illustrate our approach. Section \ref{sec:back} contains background material about credal sets and PSDDs. The technical results are presented in Section \ref{sec:csdd} where we define CSDDs, and in Sections \ref{sec:marg}-\ref{sec:map} where the three inference algorithms are derived. The results of an experimental validation are discussed in Section \ref{sec:experiments}, while conclusions and outlooks are in Section \ref{sec:conc}. Proofs are in the appendix together with some additional technical material.

\section{A Demonstrative Example}\label{sec:example}
We begin the discussion with a minimalistic example to be used as an informal introduction to the basic concepts and problems considered in the paper. Formal definitions of these basics are provided in the next section. The example is used in the other sections to demonstrate the main ideas derived in our work and show how these can be applied.

Consider four-pixel black-and-white squared images in Figure \ref{fig:squares}. These can be regarded as joint states of four Boolean variables. We assume that, out of sixteen possible configurations, only those in the top row of the Figure \ref{fig:squares} are permitted, while the remaining six in the bottom row are forbidden by some structural constraint (e.g., only ``lines'' and ``points'' can be depicted).

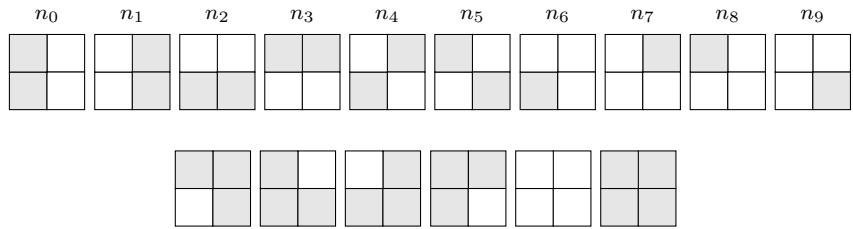
\begin{figure}[htp!]
    \centering
        \begin{tikzpicture}[scale=.5]
    \fill[black!10,draw=black] (0,0) rectangle (1,1);
    \fill[black!0,draw=black]  (1,1) rectangle (2,2);
    \fill[black!10,draw=black]  (0,1) rectangle (1,2);
    \fill[black!0,draw=black]  (1,0) rectangle (2,1);
    \node[] at (1,2.5) {\footnotesize$n_0$};
    \end{tikzpicture}
    \begin{tikzpicture}[scale=.5]
    \fill[black!0,draw=black] (0,0) rectangle (1,1);
    \fill[black!10,draw=black]  (1,1) rectangle (2,2);
    \fill[black!0,draw=black]  (0,1) rectangle (1,2);
    \fill[black!10,draw=black]  (1,0) rectangle (2,1);
    \node[] at (1,2.5) {\footnotesize$n_1$};
    \end{tikzpicture}
    \begin{tikzpicture}[scale=.5]
    \fill[black!10,draw=black] (0,0) rectangle (1,1);
    \fill[black!0,draw=black]  (1,1) rectangle (2,2);
    \fill[black!0,draw=black]  (0,1) rectangle (1,2);
    \fill[black!10,draw=black]  (1,0) rectangle (2,1);
    \node[] at (1,2.5) {\footnotesize$n_2$};
    \end{tikzpicture}
    \begin{tikzpicture}[scale=.5]
    \fill[black!0,draw=black] (0,0) rectangle (1,1);
    \fill[black!10,draw=black]  (1,1) rectangle (2,2);
    \fill[black!10,draw=black]  (0,1) rectangle (1,2);
    \fill[black!0,draw=black]  (1,0) rectangle (2,1);
    \node[] at (1,2.5) {\footnotesize$n_3$};
    \end{tikzpicture}
    \begin{tikzpicture}[scale=.5]
    \fill[black!10,draw=black] (0,0) rectangle (1,1);
    \fill[black!10,draw=black]  (1,1) rectangle (2,2);
    \fill[black!0,draw=black]  (0,1) rectangle (1,2);
    \fill[black!0,draw=black]  (1,0) rectangle (2,1);
    \node[] at (1,2.5) {\footnotesize$n_4$};
    \end{tikzpicture}
    \begin{tikzpicture}[scale=.5]
    \fill[black!0,draw=black] (0,0) rectangle (1,1);
    \fill[black!0,draw=black]  (1,1) rectangle (2,2);
    \fill[black!10,draw=black]  (0,1) rectangle (1,2);
    \fill[black!10,draw=black]  (1,0) rectangle (2,1);
    \node[] at (1,2.5) {\footnotesize$n_5$};
    \end{tikzpicture}
    \begin{tikzpicture}[scale=.5]
    \fill[black!10,draw=black] (0,0) rectangle (1,1);
    \fill[black!0,draw=black]  (1,1) rectangle (2,2);
    \fill[black!0,draw=black]  (0,1) rectangle (1,2);
    \fill[black!0,draw=black]  (1,0) rectangle (2,1);
    \node[] at (1,2.5) {\footnotesize$n_6$};
    \end{tikzpicture}
    \begin{tikzpicture}[scale=.5]
    \fill[black!0,draw=black] (0,0) rectangle (1,1);
    \fill[black!10,draw=black]  (1,1) rectangle (2,2);
    \fill[black!0,draw=black]  (0,1) rectangle (1,2);
    \fill[black!0,draw=black]  (1,0) rectangle (2,1);
    \node[] at (1,2.5) {\footnotesize$n_7$};
    \end{tikzpicture}
    \begin{tikzpicture}[scale=.5]
    \fill[black!0,draw=black] (0,0) rectangle (1,1);
    \fill[black!0,draw=black]  (1,1) rectangle (2,2);
    \fill[black!10,draw=black]  (0,1) rectangle (1,2);
    \fill[black!0,draw=black]  (1,0) rectangle (2,1);
    \node[] at (1,2.5) {\footnotesize$n_8$};
    \end{tikzpicture}
    \begin{tikzpicture}[scale=.5]
    \fill[black!0,draw=black] (0,0) rectangle (1,1);
    \fill[black!0,draw=black]  (1,1) rectangle (2,2);
    \fill[black!0,draw=black]  (0,1) rectangle (1,2);
    \fill[black!10,draw=black]  (1,0) rectangle (2,1);
    \node[] at (1,2.5) {\footnotesize$n_9$};
    \end{tikzpicture}
    \vskip 5mm
    \begin{tikzpicture}[scale=.5]
    \fill[black!0,draw=black] (0,0) rectangle (1,1);
    \fill[black!10,draw=black]  (1,1) rectangle (2,2);
    \fill[black!10,draw=black]  (0,1) rectangle (1,2);
    \fill[black!10,draw=black]  (1,0) rectangle (2,1);
    \end{tikzpicture}
    \begin{tikzpicture}[scale=.5]
    \fill[black!10,draw=black] (0,0) rectangle (1,1);
    \fill[black!0,draw=black]  (1,1) rectangle (2,2);
    \fill[black!10,draw=black]  (0,1) rectangle (1,2);
    \fill[black!10,draw=black]  (1,0) rectangle (2,1);
    \end{tikzpicture}
    \begin{tikzpicture}[scale=.5]
    \fill[black!10,draw=black] (0,0) rectangle (1,1);
    \fill[black!10,draw=black]  (1,1) rectangle (2,2);
    \fill[black!0,draw=black]  (0,1) rectangle (1,2);
    \fill[black!10,draw=black]  (1,0) rectangle (2,1);
    \end{tikzpicture}
    \begin{tikzpicture}[scale=.5]
    \fill[black!10,draw=black] (0,0) rectangle (1,1);
    \fill[black!10,draw=black]  (1,1) rectangle (2,2);
    \fill[black!10,draw=black]  (0,1) rectangle (1,2);
    \fill[black!0,draw=black]  (1,0) rectangle (2,1);
    \end{tikzpicture}
    \begin{tikzpicture}[scale=.5]
    \fill[black!0,draw=black] (0,0) rectangle (1,1);
    \fill[black!0,draw=black]  (1,1) rectangle (2,2);
    \fill[black!0,draw=black]  (0,1) rectangle (1,2);
    \fill[black!0,draw=black]  (1,0) rectangle (2,1);
    \end{tikzpicture}
    \begin{tikzpicture}[scale=.5]
    \fill[black!10,draw=black] (0,0) rectangle (1,1);
    \fill[black!10,draw=black]  (1,1) rectangle (2,2);
    \fill[black!10,draw=black]  (0,1) rectangle (1,2);
    \fill[black!10,draw=black]  (1,0) rectangle (2,1);
    \end{tikzpicture}
    \caption{Permitted (top) and un-permitted (bottom) four-pixel squared images\label{fig:squares}}
\end{figure}

Let us denote the four variables as $(X_1,X_2,X_3,X_4)$, where $X_1$ corresponds to the top-left pixel and the other ones follow a clock-wise order. If black pixel corresponds to the true state of the variable, the formula implementing the constraints can be written as:\footnote{We assume the reader to be familiar with basic propositional logic notation. More details about that can be found in Section \ref{sec:sdd}.}
\begin{equation}\label{eq:sdd_squares}
\gamma := \left[ \bigvee_{1\leq i \leq 4} X_i \right] \land  \left[ \bigvee_{1\leq  i\neq j \leq 4} \lnot X_i \wedge \lnot X_j \right] 
\end{equation}
where the two conjunctive clauses
impose, respectively, that at least one pixel is black and two
pixels are white. These
constraints rule out exactly the configurations in the bottom
row in Figure \ref{fig:squares}.

Consider the logic circuit in Figure \ref{fig:sdd}, where
conjunctive gates are depicted in blue and they alternate with
the disjunctive (red) ones. For the moment, ignore the
parameters associated with the inputs of the disjunctive gates
and the \emph{top} (i.e., $\top$) inputs of the conjunctive
ones. The reader can verify that the formula implemented by the circuit is equivalent to $\gamma$ in Equation \eqref{eq:sdd_squares}.\footnote{To see this, notice that the logic circuit in Figure \ref{fig:sdd} encodes formula 
$$\begin{array}{rll}
\phi:= & \big( (\lnot X_1 \land \lnot X_2) \land ((\lnot X_3 \land X_4) \lor X_3) \big) & \lor \\
& \big( ((X_1 \land \lnot X_2) \lor  (\lnot X_1 \land X_2) ) \land 	(( X_3 \land \lnot X_4) \lor \lnot X_3) \big) & \lor \\
& \big( X_1 \land X_2 \land \lnot X_3 \land \lnot X_4 \big) &
\end{array}$$
The three disjuncts are mutually exclusive. Models of the first disjuncts correspond to four-pixel squared images whose counts are $n_2, n_6, n_9$, models of the second disjuncts correspond to four-pixel squared images whose counts are $n_0, n_1, n_4, n_5, n_7$ and $n_8$, and finally the unique model of the third disjuncts corresponds to the four-pixel squared image whose count is $n_3$.
}

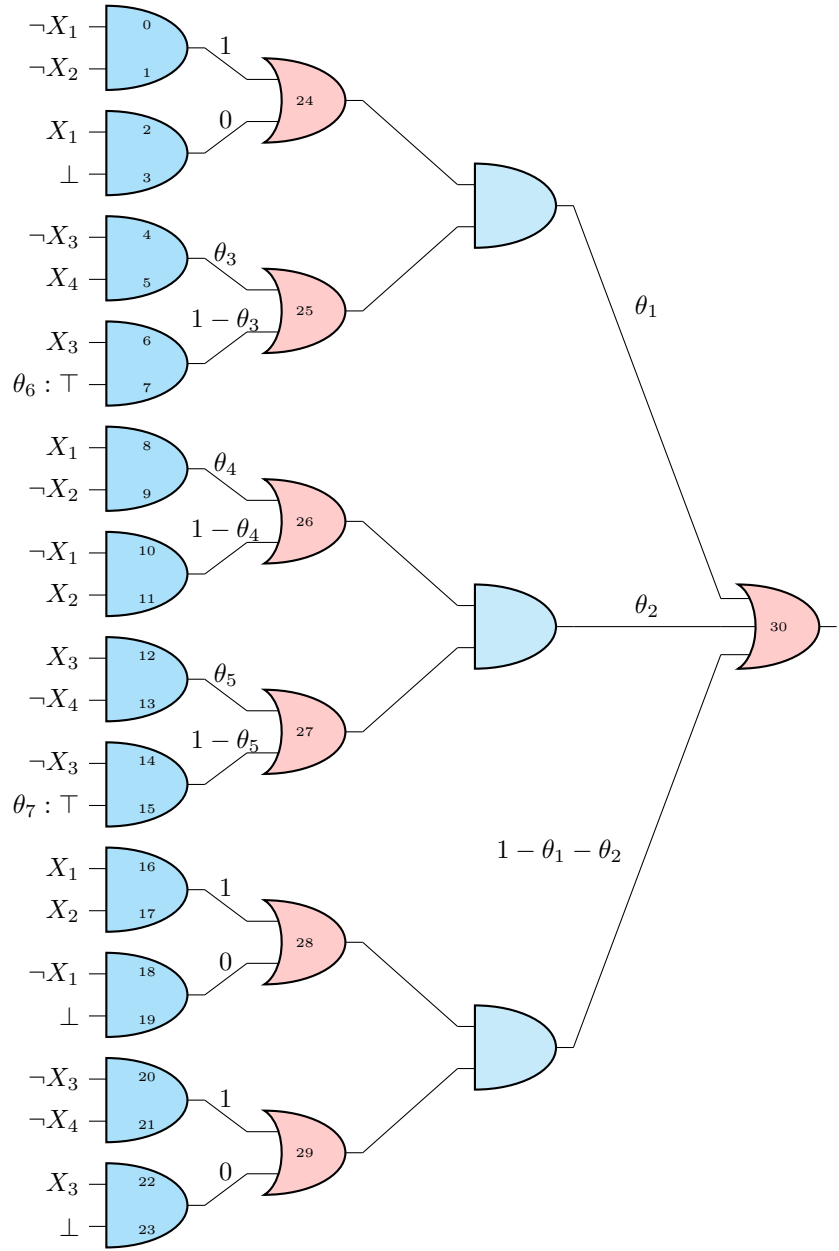
\begin{figure}
\centering
\begin{tikzpicture}[scale=.7]
\draw
(0,0) node[and port,fill=cyan!30!white] (myand1) {\tiny $\begin{array}{l}0\\\\1\end{array}$}
(0,-2) node[and port,fill=cyan!30!white] (myand2) {\tiny $\begin{array}{l}2\\\\3\end{array}$}
(0,-4) node[and port,fill=cyan!30!white] (myand3) {\tiny $\begin{array}{l}4\\\\5\end{array}$}
(0,-6) node[and port,fill=cyan!30!white] (myand4) {\tiny $\begin{array}{l}6\\\\7\end{array}$}
(0,-8) node[and port,fill=cyan!30!white] (myand5) {\tiny $\begin{array}{l}8\\\\9\end{array}$}
(0,-10) node[and port,fill=cyan!30!white] (myand6) {\tiny $\begin{array}{l}10\\\\11\end{array}$}
(0,-12) node[and port,fill=cyan!30!white] (myand7) {\tiny $\begin{array}{l}12\\\\13\end{array}$}
(0,-14) node[and port,fill=cyan!30!white] (myand8) {\tiny $\begin{array}{l}14\\\\15\end{array}$}
(0,-16) node[and port,fill=cyan!30!white] (myand9) {\tiny $\begin{array}{l}16\\\\17\end{array}$}
(0,-18) node[and port,fill=cyan!30!white] (myand10) {\tiny $\begin{array}{l}18\\\\19\end{array}$}
(0,-20) node[and port,fill=cyan!30!white] (myand11) {\tiny $\begin{array}{l}20\\\\21\end{array}$}
(0,-22) node[and port,fill=cyan!30!white] (myand12) {\tiny $\begin{array}{l}22\\\\23\end{array}$}
(3,-1) node[or port,fill=red!20!white] (myor1) {\tiny $24$}
(3,-5) node[or port,fill=red!20!white] (myor2) {\tiny $25$}
(3,-9) node[or port,fill=red!20!white] (myor3) {\tiny $26$}
(3,-13) node[or port,fill=red!20!white] (myor4) {\tiny $27$}
(3,-17) node[or port,fill=red!20!white] (myor5) {\tiny $28$}
(3,-21) node[or port,fill=red!20!white] (myor6) {\tiny $29$}
(7,-3) node[and port,fill=cyan!20!white] (myand13) {}
(7,-11) node[and port,fill=cyan!20!white] (myand14) {}
(7,-19) node[and port,fill=cyan!20!white] (myand15) {}
(12,-11) node[or port,number inputs=3,fill=red!20!white] (myor7) {\tiny $30$}
(myand1.in 1) node[anchor=east] {$\neg X_1$}
(myand1.in 2) node[anchor=east] {$\neg X_2$}
(myand2.in 1) node[anchor=east] {$X_1$}
(myand2.in 2) node[anchor=east] {$\bot$}
(myand3.in 1) node[anchor=east] {$\neg X_3$}
(myand3.in 2) node[anchor=east] {$X_4$}
(myand4.in 1) node[anchor=east] {$X_3$}
(myand4.in 2) node[anchor=east] {$\theta_{6}:\top$}
(myand5.in 1) node[anchor=east] {$X_1$}
(myand5.in 2) node[anchor=east] {$\neg X_2$}
(myand6.in 1) node[anchor=east] {$\neg X_1$}
(myand6.in 2) node[anchor=east] {$X_2$}
(myand7.in 1) node[anchor=east] {$X_3$}
(myand7.in 2) node[anchor=east] {$\neg X_4$}
(myand8.in 1) node[anchor=east] {$\neg X_3$}
(myand8.in 2) node[anchor=east] {$\theta_{7}:\top$}
(myand9.in 1) node[anchor=east] {$X_1$}
(myand9.in 2) node[anchor=east] {$X_2$}
(myand10.in 1) node[anchor=east] {$\neg X_1$}
(myand10.in 2) node[anchor=east] {$\bot$}
(myand11.in 1) node[anchor=east] {$\neg X_3$}
(myand11.in 2) node[anchor=east] {$\neg X_4$}
(myand12.in 1) node[anchor=east] {$X_3$}
(myand12.in 2) node[anchor=east] {$\bot$}
(myand1.out) -- (myor1.in 1) node[midway,above]{$1$}
(myand2.out) -- (myor1.in 2) node[midway,above]{$0$}
(myand3.out) -- (myor2.in 1) node[midway,above]{$\theta_{3}$}
(myand4.out) -- (myor2.in 2) node[midway,above=0.1cm]{$1-\theta_{3}$}
(myand5.out) -- (myor3.in 1) node[midway,above]{$\theta_{4}$}
(myand6.out) -- (myor3.in 2) node[midway,above=0.1cm]{$1-\theta_{4}$}
(myand7.out) -- (myor4.in 1) node[midway,above]{$\theta_{5}$}
(myand8.out) -- (myor4.in 2) node[midway,above=0.1cm]{$1-\theta_{5}$}
(myand9.out) -- (myor5.in 1) node[midway,above]{$1$}
(myand10.out) -- (myor5.in 2) node[midway,above]{$0$}
(myand11.out) -- (myor6.in 1) node[midway,above]{$1$}
(myand12.out) -- (myor6.in 2) node[midway,above]{$0$}
(myor1.out) -- (myand13.in 1) 
(myor2.out) -- (myand13.in 2) 
(myor3.out) -- (myand14.in 1)
(myor4.out) -- (myand14.in 2)
(myor5.out) -- (myand15.in 1)
(myor6.out) -- (myand15.in 2)
(myand13.out) -- (myor7.in 1) node[midway,above=1cm]{$\theta_1$}
(myand14.out) -- (myor7.in 2) node[midway,above]{$\theta_2$}
(myand15.out) -- (myor7.in 3) node[midway,left=.2cm]{$1-\theta_1-\theta_2$};
\end{tikzpicture}
\caption{A probabilistic sentential decision diagrams (PSDD) over four Boolean variables\label{fig:sdd}. The corresponding sentential decision diagram (SDD) is the underlying logic circuit when the probabilistic annotations of the PSDD are not considered.}
\end{figure}

Consider a data set of observations for the permitted
configurations is available, where each configuration occurs
with the counts $n_0,\dotsc,n_9$, as indicated on the top of the squares in Figure
\ref{fig:squares} for the top row.
Say that we want to learn
from these data a generative model, that is, a joint
probability mass function over the four variables.
Such a mass function should be also consistent with the
logical constraints, that is, the six impossible configurations should receive zero probability.

As the sub-formulae associated to the three inputs of the
disjunctive gate in the circuit output are disjoint, a joint
mass function consistent with $\phi$ could be simply $\theta_1
I_{\phi_1} + \theta_2 I_{\phi_2} + (1-\theta_1-\theta_2)
I_{\phi_3}$, where $\phi_i$ is the formula associated with the
$i$-th input of the gate for each $i=1,2,3$, and $I$ denotes
the indicator function of the formula in its subscript. For
each $i=1,2,3$, the parameter $\theta_i$ is therefore the
probability of $\phi_i$, that can be estimated from the data.
For example, a maximum likelihood estimator would give $\theta_1=\frac{n_2+n_6+n_9}{n}$ and $\theta_2=\frac{n_0 + n_1 + n_4 + n_5 + n_7+ n_8}{n}$ where $n=\sum_{i=0}^9 n_i$.

More refined joint mass functions can be obtained by a recursive application of this approach to the other disjunctive gates and multiplying the contributions associated with the inputs of a conjunctive gate. In those cases the parameters should be intended as conditional probabilities for the corresponding sub-formula given by a so called  \emph{context}.\footnote{Roughly, a context of a node in the circuit is the formula determined by the path leading to it and such that, joint with the underlying SDD, implies the formula associated to the node. A formal statement is given in Definition \ref{def:context}.}

Finally, for the circuit inputs, we specify indicator functions of their literals, these being replaced by a zero for \emph{bots} (i.e., $\bot$), and by a probability mass function $\theta I_X+(1-\theta) I_{\neg X}$ for a \emph{top} (i.e., $\top$) associated with variable $X$ and annotated with a probability $\theta$. Accordingly, the annotated circuit in Figure \ref{fig:sdd} induces the joint probability mass function:
\begin{multline}\label{eq:indicators}
\theta_1 \cdot \left[I_{\neg X_1} I_{\neg X_2} \right]\cdot
\left[ \theta_{3} 
I_{\neg X_3} I_{ X_4} + (1-\theta_{3}) I_{X_3}
\left[ \theta_6 I_{X_4}+\left(1-\theta_6\right)I_{\neg X_4}\right]\right] +\\
+\theta_2 \cdot 
\left[\theta_{4} I_{X_1} I_{\neg X_2} + (1-\theta_{4}) I_{\neg X_1} I_{X_2}\right] \cdot \\ 
\cdot \left[\theta_{5} I_{X_3} I_{\neg X_4} + (1-\theta_{5}) I_{\neg X_3}\left[\theta_7 I_{X_4} + (1-\theta_7) I_{\neg X_4}\right] \right] +\\ 
+(1-\theta_1-\theta_2) \cdot \left[I_{X_1}I_{X_2} \right] \cdot \left[ I_{\neg X_3} I_{\neg X_4} \right]\,,
\end{multline}
where the variables of the indicator functions are left
implicit for the sake of readability. An annotated circuit as
that in Figure \ref{fig:sdd}, defining a generative model as
the one in Equation \eqref{eq:indicators}, which is consistent
with the formula $\gamma$ in Equation \eqref{eq:sdd_squares}, is
called a \emph{probabilistic sentential decision diagram}
\cite{kisa2014}.

In this paper we are interested in developing algorithms for
sensitivity analysis of the inferences in these models with
respect to the parameters. This is important when only few training data are available and sharp estimates of the parameters might be not reliable. Moreover, the parameters not associated with the output
disjunctive gate are conditional probabilities and the closer
the parameter is to the input, the higher will be the number
of variables involved in the conditioning event. Thus, in deep
circuits, we might have very few training data to learn those
parameters even if the available training data set is huge,
thus making important the development of tools for sensitivity
analysis. The notion of probabilistic sentential decision diagrams, together with other background concepts, are formally described in the next section.
\section{Background}\label{sec:back}
\subsection{Credal Sets}\label{sec:cs}
Consider a variable $X$ taking its values in a finite set
$\mathcal{X}$ whose generic element is denoted as $x$. A
\emph{probability mass function} (PMF) over $X$, denoted as
$\mathbb{P}(X)$, is a real-valued non-negative function
$\mathbb{P}:\mathcal{X} \to \mathbb{R}$ such that
$\sum_{x\in\mathcal{X}} \mathbb{P}(x)=1$. Given a function $f$
of $X$, the \emph{expectation} of $f$ with respect to a PMF
$\mathbb{P}$ is $\mathbb{P}[f]:=\sum_{x\in\mathcal{X}} f(x)
\cdot \mathbb{P}(x)$. A set of PMFs over $X$ is called
\emph{credal set} (CS) and denoted as $\mathbb{K}(X)$. Here we
consider CSs induced by a finite number of linear constraints.
Given CS $\mathbb{K}(X)$, the bounds of the  expectation with
respect to $\mathbb{K}(X)$ can be computed by optimizing 
$\mathbb{P}[f]$ over $\mathbb{K}(X)$. For example, for the lower bound, $\underline{\mathbb{P}}[f]:= \min_{\mathbb{P}(X)\in \mathbb{K}(X)} \sum_{x\in\mathcal{X}} f(x) \cdot \mathbb{P}(x)$. This is a linear programming task, whose optimum remains the same after replacing $\mathbb{K}(X)$ with its convex hull. Such optimum is attained on an extreme point of the convex closure. Moreover, if $f$ is an indicator function, the lower expectation is called \emph{lower} probability. Notation $\overline{\mathbb{P}}$ is used instead for the upper bounds and duality $\overline{\mathbb{P}}(f)=-\underline{\mathbb{P}}(-f)$ holds.

In the special case of Boolean variables it is easy to see that the number of extreme points of the convex closure of a CS cannot be more than two, and the specification of a single interval constraint, say $0\leq l\leq\mathbb{P}(x)\leq u\leq 1$ for one of the two states is a fully general CS specification.

Learning CSs from multinomial data can be done by the \emph{imprecise Dirichlet model} (IDM) \cite{walley1996inferences}. This is a generalised Bayesian approach in which a single Dirichlet prior with equivalent sample size $s$ is replaced by the set of all the Dirichlet priors with this size. The corresponding bounds on the probabilities are
\begin{equation}\label{eq:idm}
\mathbb{P}(x) \in \left[\frac{n(x)}{N+s},\frac{n(x)+s}{N+s}\right] 
\end{equation}
where $n(x)$ are the number of instances of the data set, whose total size is $N$, such that $X=x$, for each $x\in\mathcal{X}$.

Given PMF $\mathbb{P}(X_1,X_2)$, $X_1$ and $X_2$ are \emph{stochastically independent} if and only if $\mathbb{P}(x_1,x_2)=\mathbb{P}(x_1) \cdot \mathbb{P}(x_2)$ for each $x_1\in\mathcal{X}_1$ and $x_2 \in \mathcal{X}_2$. We similarly say that, given CS $\mathbb{K}(X_1,X_2)$, $X_1$ and $X_2$ are \emph{strongly independent} if and only if stochastic independence is satisfied for each extreme point of the convex closure of the joint CS.

\subsection{Sentential Decision Diagrams}\label{sec:sdd}
Give a finite set of Boolean variables $\bm{X}$, a \emph{literal} is either a Boolean variable $X \in \bm{X}$ or its negation $\lnot X$. The Boolean constant always taking the value false or true are denoted, respectively, as $\bot$ and $\top$.

We start by defining a generalisation of orders on variables based on the following definition.

\begin{definition}[Vtree]
Consider a finite set $\bo{X}$ of Boolean variables. A
\emph{vtree} for $\bo{X}$ is a full binary tree $v$ whose
leaves are in one-to-one correspondence with the elements of
$\bo{X}$. We denote by $v^l$ (resp., $v^r$) the left (right)
subtree of $v$, i.e., the vtree rooted at the left (resp., right) child of the root of $v$.
\end{definition}

Two vtrees for the variables in the example in Section \ref{sec:example} are in Figure \ref{fig:vtree}. Note that the in-order tree traversal of a vtree induces a total order on the variables, but two distinct vtrees can induce in this way the same order (e.g., the two vtrees in Figure \ref{fig:vtree}).

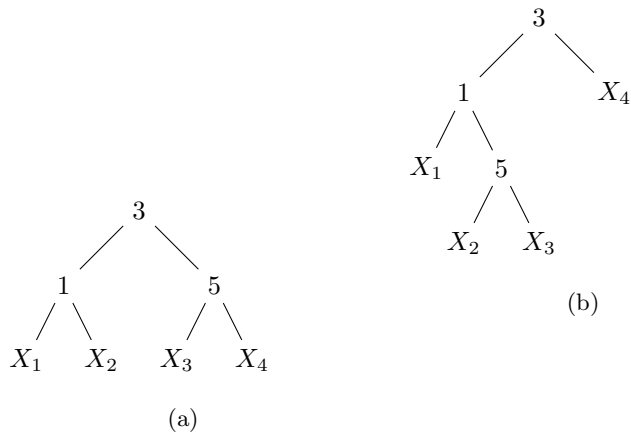
\begin{figure}[htp]\centering
\begin{subfigure}[a]{0.4\textwidth}
\begin{tikzpicture}[level distance=1cm,level 1/.style={sibling distance=2cm},level 2/.style={sibling distance=1cm}]
\node {$3$}child {node {$1$} child {node {$X_1$}}
child {node {$X_2$}}} child {node {$5$}
child {node {$X_3$}} child {node {$X_4$}}};
\end{tikzpicture}
\caption{\label{fig:vtree1}}
\end{subfigure}
\quad
\begin{subfigure}[b]{0.4\textwidth}
\begin{tikzpicture}[level distance=1cm,
level 1/.style={sibling distance=2cm},
level 2/.style={sibling distance=1cm}]
\node {$3$} child {node {$1$} child {node {$X_1$}} child { node {$5$} child {node {$X_2$}} child {node {$X_3$}}}} child {node {$X_4$}};
\end{tikzpicture}
\caption{\label{fig:vtree2}}
\end{subfigure}
\caption{Two vtrees over four variables}\label{fig:vtree}
\end{figure}

Based on the notion of vtree, we provide the following
definition of SDDs.

\begin{definition}[SDD]
A \emph{sentential decision diagram} (SDD) $\alpha$ \emph{normalised for vtree} $v$ and its interpretation $\interp{\alpha}$ are defined inductively as follows.
\begin{itemize}
\item  If $v$ is a leaf, let $X$ be the variable attached to
  $v$; then $\alpha$ is either a \emph{constant}, i.e., $\alpha \in \{ \bot,  \top \}$,  or a \emph{literal}, i.e., $\alpha \in \{X, \neg X \}$.
\item If $v$ is not a leaf, then  $\alpha = \{ (p_i,s_i)\}_{i=1}^k$, where the $p_i$'s and $s_i$'s, called \emph{primes} and \emph{subs}, are SDDs normalised for $v^l$ and $v^r$ respectively. 
\end{itemize}

The interpretation of an SDD $\alpha$  normalised for $v$, denoted as $\interp{\alpha}$, is a propositional sentence over the variables of $v$, defined as follows:
\begin{itemize}
\item If $\alpha \in \{\bot, \top, X, \neg X \}$: $\interp{\bot}= \bot$, $\interp{\top}= \top$ and $\interp{X} = X$, $\interp{\neg X} = \neg X$.
\item If $\alpha = \{ (p_i,s_i)\}_{i=1}^k $,
$\interp{\alpha}= \mathop{\bigvee}_{i=1}^k \interp{p_i } \wedge \interp{s_i}$ and interpretations $\{\interp{p_i}\}_{i=1}^k$ form a partition.
\end{itemize}
\end{definition}

The \emph{sub-SDDs} of an SDD $\alpha$ are $\alpha$ itself,
its primes, its subs, and the sub-SDDs of its primes and subs.
A sub-SDD will be often called a \emph{node}, more precisely a
\emph{terminal node} when it is normalized for a leaf, and a \emph{decision node} otherwise. 

In a decision node $\{ (p_i,s_i)\}_{i=1}^k$,  the pairs $(p_i, s_i)$'s are called the \emph{elements} of the node, and $k$ is its \emph{size}. The  size of an SDD is the sum of the sizes of all its decision nodes.\footnote{The size of an SDD depends on the number of variables, the base knowledge and the choice of the vtree. The notion of \emph{nicety} for vtrees with respect to a given formula provides a bound on the SDD size \cite{darwiche_sdd}. Yet, the existence of a nice vtree is guaranteed for CNFs only.}

At the interpretation level, each decision node represents a disjunction (actually, an exclusive disjunction, as the primes form a partition), while each of its elements is a conjunction between a prime and a sub.

\begin{example}\label{ex:1}
  Given the vtree $v$ over the ordered pair of variables $(A,P)$,
  $\alpha=\{ (A, P) , (\lnot A ,\top) \}$ is an SDD normalized
  for $v$; the interpretation of $\alpha$ is $\interp{\alpha}= (A \land
  P) \lor (\lnot A \land \top)$, which is logically equivalent to $\phi=A\rightarrow P$.
\end{example}

Given the previous discussion, 
 we can intend the SDD $\alpha$ as a rooted logic circuit, like the one in Figure \ref{fig:sdd}, providing a representation of the formula $\langle \alpha \rangle$. The labels on decision nodes denote the vtree nodes for which the decision node is normalized.

The following definition makes formal the notion of path in an SDD. This is needed to provide a semantics for the parameters used to annotate SDDs.

\begin{definition}[Context]\label{def:context} Let $n$ be a node (either terminal or decision) of an SDD. Denote as $(p_1,s_1), \ldots, (p_l,s_l)$ a path from the root to node $n$. Then the conjunction of the interpretations of the primes encountered in this path, i.e., $\langle p_1\rangle \wedge \dots \wedge \langle p_l \rangle$, is called a \emph{context} of $n$ and denoted as $\gamma_n$. The context $\gamma_n$ is \emph{feasible} if and only if $s_i \neq \bot$ for each $i=1,\ldots,l$.
\end{definition}

By construction, each node has at least one context. The number of contexts of a node defines its \textit{multiplicity} as follows.

\begin{definition}\label{def:topology} 
The multiplicity of an SDD node is the number of its contexts. An SDD is \textit{singly connected} if all of its nodes have multiplicity equal to one. Otherwise, it is \textit{multiply connected}.
\end{definition}

Notice that, at the circuit level, the definition of multiply connected SDD coincides with the graph-theoretical one.

\begin{figure}
\centering
\begin{tikzpicture}[scale=.7]
\draw
(-1,4) node[and port,fill=cyan!30!white] (myand2) {\tiny $\begin{array}{l}0\\\\1\end{array}$}
(-1,-1) node[and port,fill=cyan!30!white] (myand1) {\tiny $\begin{array}{l}2\\\\3\end{array}$}
(1,1) node[or port,fill=red!20!white] (myor1) {\tiny $4$}
(4,-2) node[and port,fill=cyan!30!white] (myand3)  {\tiny $\begin{array}{l}9\\\\10\end{array}$}
(5,1) node[and port,fill=cyan!30!white] (myand4) {\tiny $\begin{array}{l}8\\\\\phantom{a}\end{array}$}
(5,3) node[and port,fill=cyan!30!white] (myand5)  {\tiny $\begin{array}{l}6\\\\7\end{array}$}
(4,6) node[and port,fill=cyan!30!white] (myand6)  {\tiny $\begin{array}{l}5\\\\\phantom{a}\end{array}$}
(7,0) node[or port,fill=red!20!white] (myor2) {\tiny $13$}
(7,+5) node[or port,fill=red!20!white] (myor3) {\tiny $11$}
(10,0) node[and port,fill=cyan!30!white] (myand7)  {\tiny $\begin{array}{l}\phantom{a}\\\\14\end{array}$}
(10,4) node[and port,fill=cyan!30!white] (myand8)  {\tiny $\begin{array}{l}\phantom{a}\\\\12\end{array}$}
(13,2) node[or port,fill=red!20!white] (myor4) {\tiny$15$}
(myand1.in 1) node[anchor=east] {$X_2$}
(myand1.in 2) node[anchor=east] {$\theta_1:\top$}
(myand2.in 1) node[anchor=east] {$\neg X_2$}
(myand2.in 2) node[anchor=east] {$\bot$}
(myand3.in 1) node[anchor=east] {$X_1$}
(myand3.in 2) node[anchor=east] {$\bot$}
(myand4.in 1) node[anchor=east] {$\neg X_1$}
(myand5.in 1) node[anchor=east] {$\neg X_1$}
(myand6.in 1) node[anchor=east] {$X_1$}
(myand5.in 2) node[anchor=east] {$\bot$}
(myand7.in 2) node[anchor=east] {$\neg X_4$}
(myand8.in 2) node[anchor=east] {$X_4$}
(myand1.out) -- (myor1.in 2) node[midway,left=0.1cm]{$1$}
(myand2.out) -- (myor1.in 1) node[midway,below left=0.1cm]{$0$}
(myor1.out) -- (myand6.in 2) node[midway,above]{}
(myor1.out) -- (myand4.in 2) node[midway,above]{}
(myand3.out) -- (myor2.in 2) node[midway,above left=0.1cm]{$0$}
(myand4.out) -- (myor2.in 1) node[midway,above=0.2cm]{$1$}
(myand5.out) -- (myor3.in 2) node[midway,above left=0.1cm]{$0$}
(myand6.out) -- (myor3.in 1) node[midway,above=0.2cm]{$1$}
(myor2.out) -- (myand7.in 1) node[midway,above]{}
(myor3.out) -- (myand8.in 1) node[midway,above]{}
(myand7.out) -- (myor4.in 2) node[midway,below right=0.2cm]{$1-\theta_4$}
(myand8.out) -- (myor4.in 1) node[midway,above=0.2cm]{$\theta_4$};
\end{tikzpicture}
\caption{A PSDD whose underlying SDD is multiply connected,  normalized for the second vtree in Figure \ref{fig:vtree}, and represents formula $\phi = (X_1\wedge X_2 \wedge X_4)\vee (\neg X_1 \wedge X_2 \wedge \neg X_4)$.}
\label{fig:multi}
\end{figure}
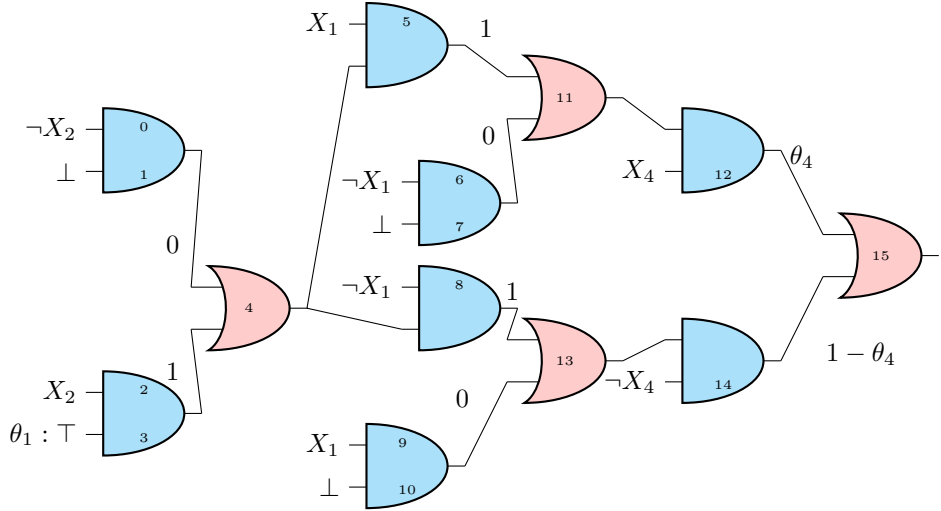

\begin{example}
Consider SDD in Figure \ref{fig:multi}. The terminal node with
label $12$ has multiplicity one and its context is $\gamma =
X_1 \wedge X_2$. The decision node with label $4$ (in pink in
the figure) has multiplicity two and its contexts are $\gamma'= ((X_1 \land X_2) \land X_1) = X_1\wedge X_2$ and $\gamma'':= ((\neg X_1 \land X_2) \land \neg X_1) = \neg X_1\wedge X_2$.
\end{example}

The interpretation of a node is implied by its contexts and by
the interpretation of the SDD it belongs to, that is,
for each node $n$ of an SDD $\alpha$, for any  context
$\gamma_n$, we have that $\interp{\alpha} \wedge  \gamma_n  \models \interp{n}$.

Let us finally define a notion of topological order for the nodes of an SDD. The logic circuit underlying the SDD can be regarded as a directed graph whose arcs are oriented from the inputs to the outputs. Yet, an order in the circuit does not induce a complete order over the SDD nodes as the conjunctive gates corresponds to pairs or nodes (i.e., elements). Nevertheless, to obtain a complete order we might simply force both the nodes of an element to precede their decision node, while the terminal nodes are clearly preceding all the decision nodes.

\subsection{Probabilistic Sentential Decision Diagrams}\label{sec:psdd}
A \emph{probabilistic sentential decision diagram} is a parametrized SDD, where parameters are PMFs specifications on the decision nodes and on the terminal nodes labeled with  constant top. A PSDD induces a joint PMF over its variables, assigning zero probability to the impossible states of the logical constraint given by the interpretation of the underlying SDD.

To turn an SDD into a PSDD, proceed as follows. For each
terminal node $\top$, specify a positive parameter $\theta$
such that $0\leq \theta \leq 1$. Notation for such terminal
node is $X : \theta$, where $X$ is the variable of the leaf
vtree node for which $\top$ is normalised. Terminal nodes
other than $\top$ appear as they are; for each decision node
$\{ (p_i,s_i)\}_{i=1}^k$, specify for each prime $p_i$ a real
number $\theta_i \geq 0$, such that $\sum_{i=1}^k\theta_i=1$
and $\theta_i = 0$ if and only if $s_i=\bot$. Notation
$\{(p_i,s_i, \theta_i)\}_{i=1}^k$ is used to denote such a parametrisation.
The interpretation of such parametrisation is the following.
Each node $n\neq \bot$ normalized for vtree node $v$ induces a
PMF $\mathbb{P}_n$ defined inductively as follows:

\begin{itemize}
\item if $n$ is a terminal node whose corresponding variable
  in $v$ is $X$, then $\mathbb{P}_n$ is a PMF over $\{\top,
  \bot \}$ such that:

\begin{itemize}
\item if $n=X$, $\mathbb{P}_n(\top)=1$ and $\mathbb{P}_n(\bot)=0$
\item if $n=\neg X$, $\mathbb{P}_n(\top)=0$ and $\mathbb{P}_n(\bot)=1$
\item if $n= X : \theta$, $\mathbb{P}_n(\top)=\theta$ and $\mathbb{P}_n(\bot)=1-\theta$
\end{itemize}

\item if $n=\{(p_i,s_i, \theta_i)\}_{i=1}^k$ is a decision node, let $(\bm{X},\bm{Y})$ be the variables of $v^l$, $v^r$ respectively. Then the joint PMF $\mathbb{P}_n(\bm{X},\bm{Y})$ is defined as:
\begin{equation}
\mathbb{P}_n(\bm{x},\bm{y}):=  \mathbb{P}_{p_i}(\bm{x})\cdot  \mathbb{P}_{s_i}(\bm{y})\cdot \theta_i\,,
\end{equation}
for each $(\bm{x},\bm{y}) \in \bm{\mathcal{X}} \times \bm{\mathcal{Y}}$, where $i$ is the unique index such that $\bo{x}\models \interp{p_i}$.
\end{itemize}

In other words, PSDDs are SDDs with PMFs associated to each node distinct from $\bot$. It follows that sub-SDDs of a PSDD are in fact sub-PSDDs, except for terminal nodes $\bot$ (because such nodes do not induce a PMF). According to the \emph{Base Theorem} for PSDDs \cite[Theorem 1]{kisa2014}, the PMF $\pn$ assigns zero probability to events which do not respect the propositional sentence associated to the SDD $n$. More precisely, for any instantiation $(\bm{x},\bm{y})$ of variables $(\bm{X},\bm{Y})$ of the vtree $n$ is normalised for, $\pn(\bm{x},\bm{y})> 0$ iff $(\bm{x},\bm{y})\models \interp{n}$. Moreover, the probabilities $\pn(\interp{p_i})$ are the parameters $\theta_i$'s of $n = \{(p_i,s_i, \theta_i)\}_{i=1}^k$.

We simply denote as $\mathbb{P}$ the (joint) PMF induced by the root $r$. PMF $\pn$ induced by an internal node can be obtained by conditioning $\mathbb{P}$ on a feasible context of the considered node \cite[Theorem 4]{kisa2014}: for each feasible context $\gamma_n$ of $n$, $\pn(\cdot) = \mathbb{P}(\cdot \vert \gamma_n)$. The topological definitions made for SDDs extend to PSDDs. Finally, we have the following result about independence \cite[Theorem 5]{kisa2014}: according to $\mathbb{P}$, the variables inside $v$ are independent of those outside $v$ given context $\gamma_n$. This is the PSDD analogue of the \emph{Markov condition} for Bayesian networks.


\subsection{Inferences in PSDDs}
PSDD inferences are computed with respect to the joint PMF
$\mathbb{P}$. The probability of a joint state $\bm{e}$ of a
set of PSDD variables $\bm{E}$ can be obtained in linear time
with respect to the diagram size by the bottom-up (i.e., based
on a topological order from the inputs to the output) scheme in Algorithm~\ref{alg:kisa}. Note that here and in the rest of the paper we assume that the nodes of the PSDD are labeled by integers from one to $N$ following a topological order and $N$ is therefore the output/root of the circuit. Given a vtree node $v$, notation $\bo{e}_v$ is used for the subset of $\bo{e}$ including only the variables of $v$. Note also that, as the node index $n$ in the loop follows a topological order, the \emph{message} $\pi(n)$, to be computed after the \emph{else} statement, is always a combination of messages already computed.

\begin{algorithm}[htp!]
    \caption{Probability of evidence \cite{kisa2014}}
    \label{alg:kisa}
    \begin{algorithmic}
        \STATE {\bf input:} PSDD, {evidence $\mathbf{e}$}
        \FOR{$n \leftarrow 1,\ldots,N$ (topological order)}
        \STATE{ $\pi(n) \leftarrow 0$}
        \IF{node $n$ is terminal, $n\neq \bot$}
        \STATE{$v\leftarrow $ leaf vtree node that $n$ is normalized for}
        \STATE{$\pi(n)\leftarrow \mathbb{P}_n(\mathbf{e}_v)$}
        \ELSE
        \STATE{ ${(p_i,s_i,\theta_i)}_{i=1}^k \leftarrow n$ (decision node)}
        \STATE{ $\pi(n)\leftarrow  \sum_{i=1}^k \pi(p_i) \cdot \pi(s_i) \cdot \theta_i$}
        \ENDIF
        \ENDFOR
        \STATE {\bf output:} $\mathbb{P}(\mathbf{e})\leftarrow \pi(N)$
    \end{algorithmic}
\end{algorithm}

The computation of a conditional query is based on a similar
strategy.

Regarding MAP inference, that is, the problem of finding the
most probable configuration for a set of variables given an
observation of the other ones,
the computation proceeds very similarly, replacing the sums
with maximizations  \cite{bekker2015tractable}. More formally,
given a
PSDD rooted at $r$, and evidence $\bm{e}$ for the variables in
$\bm{E}$, we are interested in finding
$\bm{x}^*:=\arg\max_{\bm{x}\in\bm{\mathcal{X}}}
\mathbb{P}_r(\bm{x}|\bm{e})$ for the PSDD variables other than
$\bm{E}$ and denoted as $\bm{X}$.
We assume the evidence consistent with the PSDD logical
constraints and hence $\mathbb{P}_r(\bm{e})>0$. This way, the
task is well-defined and it is equivalent to the maximization
of the joint, that is,
\begin{equation}\label{eq:map}
\bm{x}^* = \arg\max_{\bm{x}\in\bm{\mathcal{X}}} \mathbb{P}_r(\bm{x},\bm{e})\,.
\end{equation}

Algorithm \ref{alg:map} takes as input a PSDD rooted at $r$ over variables $\{\bo{X}, \bo{E}\}$ (with $\bo{X}$ and $\bo{E}$ disjoint) and evidence $\bo{e}$ over variables $\bo{E}$, and computes $\mathop{\max}_{\bo{x}\in \bm{\mathcal{X}}}\mathbb{P}_r(\bo{x},\bm{e})$.
Correctness is implied by the following result.
\begin{theorem}\label{theo:map}
The output of Algorithm \ref{alg:map} is the probability of
the configuration of Equation \eqref{eq:map}, that is,
\begin{equation}
MAP(r) = \mathbb{P}_r(\bm{x}^*,\bm{e})\,.
\end{equation}
\end{theorem}
Finally, the arguments realizing the maximum may be obtained by backtracking the solutions of the maximizations.
\begin{algorithm}[htp!]
    \caption{MAP}
    \label{alg:map}
    \begin{algorithmic}
        \STATE {\bo{input}:} PSDD $r$, {evidence $\mathbf{e}$}
        \FOR{$n \leftarrow 1,\ldots,N$}
        \STATE{ $MAP(n) \leftarrow 0$}
        \IF{node $n$ is terminal, $n\neq \bot$}
        \STATE{$v\leftarrow $ leaf vtree node that $n$ is normalized for}
        \IF{$var(v)\in \bo{X}$}
        \IF{$n \in \{ X, \neg X \}$} 
        \STATE{$MAP(n)\leftarrow 1$}
        \ELSIF{$n= (\top, \theta)$}
        \STATE{$MAP(n)\leftarrow \max\{\theta, 1-\theta\}$}
        \ENDIF
        \ELSE
        \STATE{$MAP(n)\leftarrow \mathbb{P}_n(\bo{e})$}
        \ENDIF
        \ELSE
        \STATE{ ${(p_i,s_i,\theta_i)}_{i=1}^k \leftarrow n$ (decision node)}
        \STATE{ $MAP(n)\leftarrow  \max_{i=1}^k MAP(p_i) \cdot MAP(s_i) \cdot \theta_i$}
        \ENDIF
        \ENDFOR
        \STATE {\bo{output}:} $MAP(N)$
    \end{algorithmic}
\end{algorithm}

\section{Credal Sentential Decision Diagrams}\label{sec:csdd}
In this section we present a generalization of PSDDs (see Section \ref{sec:psdd}) based on the notion of credal set provided in Section \ref{sec:cs}. The number of variables involved in a node's context increases with the distance from the root when the SDD is singly connected (see Definition \ref{def:context}).  As the PMFs associated with decision nodes specify probabilities conditional on the (unique) corresponding context, the amount of data used to estimate such parameters decreases rapidly with the ``depth'' of the node. In the case of a multiply connected circuit, deepest nodes with high multiplicity generally do not suffer from data scarcity, thanks to their multiple contexts. Nevertheless, data scarcity can affect single-multiplicity nodes in multiply connected circuits, namely when a deep, singly-connected sub-circuit is present. This justifies the need of a robust statistical learning of the parameters as the one provided by the IDM, even when data is initially abundant. This motivates the following definition of CSDDs.

\begin{definition}
A \emph{credal sentential decision diagram} (CSDD) is an SDD augmented as follows.
\begin{itemize}
\item For each terminal node $\top$, an interval $[l, u]$ is provided such that $0<l\leq u<1$. Notation $X:[l,u]$, where $X$ is the variable of the leaf vtree node that $\top$ is normalised for, is consequently adopted. Terminal nodes other than $\top$ appear as they are.
\item For each decision node $n = \{ (p_i,s_i)\}_{i=1}^k$, a CS $\mathbb{K}_n(P)$ is provided over a variable $P$, whose states are the interpretations  $\langle p_i \rangle$ of the primes $p_i$'s of $n$. We require that for all $\mathbb{P}(P) \in \mathbb{K}_n(P)$, for each $1\leq i \leq k$, $\mathbb{P}(\langle p_i \rangle )=0$ if and only if $s_i = \bot$.
\end{itemize}
\end{definition}
According to the above definition, the CSs associated with the decision nodes assign strictly positive (lower) probability to all the states of $P$ apart from those corresponding to a prime whose sub is $\bot$. Similarly, the intervals $[l,u]$ assigned to terminal nodes $\top$ are also CS specifications (see Section \ref{sec:cs}), while literal terminal nodes have attached degenerate CSs containing the single PMF induced by the same literal when regarded as a PSDD node. It follows that sub-SDDs different from $\bot$ (with their CSs) are in fact sub-CSDDs. Thanks to this requirement, it follows that each assignment of the parameters respecting the CSDD constraints defines a \textit{compatible} PSDD. Thus, the interpretation of a CSDD is a collection of PSDDs compatible with its constraints. This also gives a semantics for the CSDD CSs, which are regarded as conditional CSs for the variables/events in the associated nodes given a context.

Exactly as a PSDD defines a joint PMF, a CSDD defines a joint CS. Such a CS, called here the \emph{strong extension} of the CSDD and denoted as $\mathbb{K}^r(\bm{X})$, where $r$ is the root node of the CSDD, is defined as the convex hull of the set of joint PMFs  induced by the collection of its compatible PSDDs. By definition of CSDD strong extension and by the Base Theorem for PSDDs, we have the following result.

\begin{theorem}[Base]\label{theo:base}
For each node $n$ of a CSDD, for each instantiation $\bo{z}$ of its variables $\bo{Z}$, 
\begin{align}
\underline{\mathbb{P}}_n(\bm{z})>0  & \quad\text{iff}\quad  \bm{z}\models \interp{n}\,,\\
\overline{\mathbb{P}}_n(\bm{z})=0 	& \quad\text{iff}\quad  \bm{z}\not\models \interp{n} \,,
\end{align}
where $\underline{\mathbb{P}}_n(\bm{z}) = \min_{\mathbb{P}(\bo{Z}) \in \mathbb{K}^n(\bo{Z})} \mathbb{P}(\bo{z})$ and $\overline{\mathbb{P}}_n(\bm{z}) = \max_{\mathbb{P}(\bo{Z}) \in \mathbb{K}^n(\bo{Z})} \mathbb{P}(\bo{z})$.
\end{theorem}

\begin{example}\label{ex:toy_IDM}
Consider the PSDD in Figure \ref{fig:sdd}. This model can be converted into a CSDD by simply replacing the (precise) learning of the parameters from the data set of consistent observations in Figure \ref{fig:squares} with IDM-based (see Section \ref{sec:cs}) interval-valued estimates. The intervals associated with two of the seven parameters are:
\begin{align}
\theta_1 &= P(\neg X_1 \wedge \neg X_2) \in \left[ \frac{n_2+n_6+n_9}{n+s},\frac{n_2+n_6+n_9+s}{n+s} \right]\\
\theta_6 &= P(X_4|(\neg X_1 \wedge \neg X_2)\wedge X_3)\in \left[ \frac{n_2}{n_2+n_9+s},\frac{n_2+s}{n_2+n_9+s}\right]\,,
\end{align}
while the complete set of constraints on the parameters is in the appendix.
\end{example}

As in PSDDs, the CSs of a CSDD are associated with conditional
probabilities based on a context, which for ``deep'' nodes
are estimated from small amounts of data consistent with the
context; the use of robust estimators such as the IDM allows
for CS size to be proportional to the amount of data (see
Section \ref{sec:cs}), which leads to more conservative
inferences.

Inference in a CSDD is intended as the computation of lower and upper bounds with respect to its strong extension. An important remark is that, as the extreme points of the convex hull of a set also belong to the original set, the extreme points of the strong extension are joint PMFs induced by PSDDs (whose local PMFs are compatible with the local CSs in the CSDD). As a consequence of that, a CSDD encodes the same probabilistic independence relations of a PSDD  with the same underlying SDD, but based on the notion of strong independence instead of that of stochastic independence (see Section \ref{sec:cs}). Thus, the variables of a node are \emph{strongly} independent from the ones outside the node when its context is given and feasible. In this sense, the relation between PSDDs and CSDDs retraces that between BNs and credal networks \cite{cozman2000}. In the next three sections we address the problem of computing inferences in CSDDs.
\section{Marginal Inference in CSDDs}\label{sec:marg}
Recall that Algorithm \ref{alg:kisa} computes the probability of a marginal query in a PSDD. Algorithm \ref{alg:lowerlik} provides an extension of this procedure to CSDDs, allowing for the computation of lower/upper marginal probabilities. The procedure follows exactly the same scheme based on a topological order. Unlike Algorithm \ref{alg:kisa}, every time a decision node is processed, Algorithm \ref{alg:lowerlik} requires the solution of a linear programming task whose feasible region is the CS associated with the decision node.

\begin{algorithm}[htp]
    \caption{Lower probability of evidence}
    \label{alg:lowerlik}
    \begin{algorithmic} 
        \STATE {\bf input:} CSDD, evidence $\mathbf{e}$
        \FOR{$n \leftarrow 1,\ldots,N$}
        \STATE{ $\underline{\pi}(n) \leftarrow 0$}
        \IF{$n$ is terminal, $n\neq \bot$}
        \STATE{$v\leftarrow $ leaf vtree node that $n$ is normalized for}
        \STATE{$\underline{\pi}(n) \leftarrow \underline{\mathbb{P}}_n(\mathbf{e}_v)$}
        \ELSE
        \STATE{ $({(p_i,s_i)}_{i=1}^k,\mathbb{K}_n(P)) \leftarrow n$ (decision node)}
        \STATE{ $\underline{\pi}(n) \leftarrow  \min_{[\theta_1,\dots,\theta_k] \in \mathbb{K}_n(P)} \sum_{i=1}^k \underline{\pi}(p_i) \cdot \underline{\pi}(s_i) \cdot \theta_i$}
        \ENDIF
        \ENDFOR
        \STATE {\bf output:} $\underline{\mathbb{P}}(\mathbf{e})\leftarrow\underline{\pi}(N)$ 
    \end{algorithmic}
\end{algorithm}

To see why the algorithm properly computes $\underline{\mathbb{P}}(\mathbf{e})$ just regard the output of Algorithm \ref{alg:kisa} as a symbolic expression of the local probabilities involved in the CSDD local CSs. This is a multi-linear function of these probabilities subject to the linear constraints defining the CSs. The optimizations with respect to the CSs of the terminal nodes can be done independently of the others, and in any order. Afterwards, the decision nodes whose primes and subs are (already processed) terminal nodes can be safely processed too. In turn, decision nodes whose primes and subs are already processed terminal or decision nodes can be safely processed as well, and so on. Any topological order respects such priorities. The algorithm runs in polynomial time with respect to the SDD size, as it requires the solution of a single linear programming task for each CS of the CSDD. Note that for terminal nodes the optimization is trivial as it only consists in the computation of a lower probability for a CS over a Boolean variable. An analogous procedure can also be defined for upper probabilities.

The intuition above is made formal by the next theorem, stating that the output of Algorithm \ref{alg:lowerlik} is indeed the lower bound of a query with respect to the strong extension of the CSDD.

\begin{theorem}\label{theo:evidence}
Consider a CSDD and a node $n\neq \bot$ normalized for vtree $v$ with variables $\bo{Z}$. Let $\bo{e}$ be a partial or total evidence over variables in $\bo{Z}$ :
\begin{equation}
\underline{\pi}(n)= \lo{\mathbb{P}}_n(\bo{e})\,.
\end{equation}

where $\underline{\pi}(n)$ is the message associated to node $n$ by Algorithm \ref{alg:lowerlik}
\end{theorem}

In the above theorem, there are no restrictions on the topology of the CSDD. Indeed, for any node $n$, the computation of $\lo{\pi}(n)$ only depends on $n$'s predecessors with respect to a topological order. To make this clear, assume that the CSDD is multiply connected, i.e., that there exist two distinct decision nodes $n$ and $n'$ sharing a sub-CSDD $m$, say in the $i$-th , respectively $j$-th element\footnote{The case in which two nodes $n$ and $n'$ share a common sub-CSDD possibly lower than a prime or sub relies on the one treated here.}. Then $m$ is a predecessor of both $n$ and $n'$. Hence, $\lo{\pi}(m)$ will be already computed when the algorithm is about to compute $\lo{\pi}(n)$ and $\lo{\pi}(n')$, and will appear in the computations of the latter as a factor of the $i^{th}$, respectively $j^{th}$ coefficient of two LPs over distinct local CSs attached to $n, n'$ respectively. This means that the optimal configuration of $m$ will not be modified in any manner during the optimizations relative to $n$ and $n'$, and so multiply connectedness does not compromise the operations of Algorithm \ref{alg:lowerlik}.

\begin{example}\label{ex:marginal}
As an example of application of Algorithm \ref{alg:lowerlik}, assume the counts for the observations of the ten permitted four-pixel images in Figure \ref{fig:squares} are $n_0 = 30$, $n_1 = 8$, $n_2 = 5$, $n_3 = 17$, $n_4 = 3$, $n_5 = 0$, $n_6 = 12$, $n_7 = 2$, $n_8 = 9$, and $n_9 = 14$, this leading to a total of $n=100$ observations. Using the IDM with $s=1$, the PSDD in Figure \ref{fig:sdd} becomes a CSDD whose parameters are constrained by the following constraints:
\begin{align*}
\theta_1  \in \left[ \frac{31}{101},\frac{32}{101} \right] , \theta_2  \in \left[ \frac{52}{101},\frac{53}{101}\right], \theta_3  \in \left[ \frac{12}{32},\frac{13}{32}\right],\\\theta_4  \in \left[ \frac{39}{53},\frac{40}{53}\right], \theta_5  \in \left[ \frac{8}{53},\frac{9}{53}\right],
\theta_6  \in \left[ \frac{5}{20},\frac{6}{20}\right], \theta_7  \in \left[ \frac{33}{45},\frac{34}{45}\right]\,.
\end{align*}
Consider a complete evidence $(X_1=\bot,X_2=\bot,X_3=\bot,X_4=\top)$. The output of Algorithm \ref{alg:lowerlik} corresponds to the following minimization:
\begin{equation}\min_{\substack{\theta_1\in \left[ \frac{31}{101},\frac{32}{101} \right]\\ \theta_2  \in \left[ \frac{52}{101},\frac{53}{101}\right]}}\lo{\pi}(24)\cdot\lo{\pi}(25)\cdot \theta_1 + \lo{\pi}(26)\cdot\lo{\pi}(27)\cdot \theta_2 + \lo{\pi}(28)\cdot\lo{\pi}(29)\cdot( 1- \theta_1 - \theta_2)\,,
\end{equation}
where $\underline{\pi}(24)$ requires no minimization because of the sharp parameters on the arcs of node $24$ and has therefore value $\lo{\pi}(0)\cdot\lo{\pi}(1)\cdot 1 = 1$, while
\begin{equation}\label{eq:minim}
\lo{\pi}(25) = \min_{\theta_3  \in \left[ \frac{12}{32},\frac{13}{32}\right]} \lo{\pi}(4)\cdot\lo{\pi}(5)\cdot \theta_3 + \lo{\pi}(6)\cdot\lo{\pi}(7)\cdot (1-\theta_3)\,.
\end{equation}
As $\lo{\pi}(6)= 0$ and $\lo{\pi}(4)\cdot\lo{\pi}(5)=1\cdot 1 = 1$ the result of the minimization in Equation \eqref{eq:minim} is $\frac{12}{32}$. It is an easy exercise to verify that both $ \lo{\pi}(26)$ and $\lo{\pi}(28)$ are equal to zero. It follows that the output $\lo{\pi}(15)$, i.e. the lower probability  $\lo{\mathbb{P}}(X_1= \bot, X_2 = \bot, X_3= \bot, X_4 = \top)$ has value $\frac{12}{32} \cdot \frac{31}{101} \simeq 0.1151$. Note that the complete evidence considered in this example corresponds to the four-pixel image in Figure \ref{fig:squares} whose count is $n_6$, and value returned for the lower probability looks reasonably consistent with the maximum likelihood estimate $\frac{n_6}{n} = \frac{12}{100}$.
\end{example}
\section{Conditional Queries in CSDDs}\label{sec:cond}
In the previous section we discussed the computation by Algorithm \ref{alg:lowerlik} of lower (or upper) marginal probabilities in a CSDD. This corresponds to a sequence of linear programming tasks whose feasible regions are the CSs of the CSDD processed in topological order, thus taking polynomial time with respect to the diagram size. In this section we show that something similar can also be done for conditional queries.

Let $X=x$ denote the variable and state to be queried, and let $\bm{e}$ be the available evidence about other variables in a CSDD $\alpha$ rooted at $r$ with variables $\bm{X}$. 
The task is to compute the lower conditional probability with respect to the strong extension, i.e., 
\begin{equation}\label{eq:gbr}
\underline{\mathbb{P}}(x|\bm{e}) = \min_{\mathbb{P}(\bm{X}) \in \mathbb{K}^r(\bm{X})} \frac{\mathbb{P}(x,\bm{e})}{\mathbb{P}(\bm{e})}\,.
\end{equation}
To have $\underline{\mathbb{P}}(x|\bm{e}) $ well defined, we assume $\bm{e}$ to be consistent with the underlying SDD interpretation $\interp{\alpha}$. To see this, assume there is a total instantiation of $\bm{X}$ extending $\bm{e}$. Then, given an extreme point $\mathbb{P}(\bm{X})$ of the strong extension $\mathbb{K}^r(\bm{X})$, the Base Theorem for PSDDs tells us that $\mathbb{P}( \bm{x}) > 0$ if and only if $ \bm{x} \models \interp{\alpha}$. This immediately yields that the denominator in the right-hand side of Equation \eqref{eq:gbr} is positive for each  extreme point of the strong extension $\mathbb{K}^r(\bm{X})$ if and only if $\bm{e}$ is consistent with $\interp{\alpha}$.

Note also that if $\bm{e} \models \lnot x$, then $\mathbb{P}(x, \bm{e})=0$, and similarly if $\bm{e} \models x$, then $\mathbb{P}(\lnot x, \bm{e})=0$. 
Otherwise both $(x, \bm{e}) = (x, \bm{e}_{\overline{v}})$ and $(\lnot x, \bm{e}) = (\lnot x, \bm{e}_{\overline{v}})$, and therefore $\mathbb{P}(x, \bm{e})=\mathbb{P}(x, \bm{e}_{\overline{v}})$ and $\mathbb{P}(\lnot x, \bm{e})=\mathbb{P}(\lnot x, \bm{e}_{\overline{v}})$, 
where $v$ is the leaf node with variable $X$ in the vtree the CSDD is normalized for. In the following we might therefore assume $\bm{e}_{\overline{v}}=\bm{e}$.

The task in Equation \eqref{eq:gbr} corresponds to the linearly constrained minimization of a (multilinear) fractional function of the probabilities. This prevents a straightforward application of the same approach considered in the previous section. Thus, we consider instead a decision version of the optimization task in Equation \eqref{eq:gbr}, i.e., deciding whether or not the following inequality is satisfied for a given $\mu\in [0,1]$:
\begin{equation}\label{eq:gbr2}
\underline{\mathbb{P}}(x|\bm{e}) > \mu\,.
\end{equation}

As for the algorithm in \cite{decooman2010epistemic}, an algorithm able to solve Equation \eqref{eq:gbr2} for any $\mu\in [0,1]$ inside a bracketing scheme linearly converges to the actual value of the lower probability. 

As $\mathbb{P}(x|\bm{e})+\mathbb{P}(\neg x|\bm{e})=1$ for each $\mathbb{P}(\bm{X}) \in \mathbb{K}^r(\bm{X})$, and assuming that $\mathbb{P}(\bm{e})>0$, Equation \eqref{eq:gbr2} holds if and only if the following inequality holds:
\begin{equation}\label{eq:gbr3}
\min_{\mathbb{P}(\bm{X}) \in \mathbb{K}^r(\bm{X})}  \left[ (1-\mu) \mathbb{P}(x,\bm{e})-\mu \mathbb{P}(\neg x,\bm{e}) \right] > 0\,.
\end{equation}

In order to define an algorithm solving the task of deciding whether or not  inequality \eqref{eq:gbr3} is satisfied for a given $\mu \in [0,1]$ we need to define the following auxiliary quantities.
\begin{enumerate} 
\item[(i)] For a given value of $\mu$ and any node $n\neq \bot$ normalized for vtree node $v$:
\begin{equation}
\underline{\rho}_n(\mu):= (1-2\mu)\cdot \underline{\mathbb{P}}_n(\bm{e}_v)\,.
\end{equation}
\item[(ii)] For a given value of $\mu$ and a terminal node $n\neq \bot$:
\begin{equation}\label{eq:gamma}
\Lambda_{n}(\mu):=
\left\{
\begin{array}{ll}
\lambda_{n}(\mu) & \mathrm{if}\,\, $X$ \text{ occurs in } n \\
\underline{\rho}_n(\mu) & \mathrm{otherwise}\,,
\end{array}
\right.
\end{equation}
with
\begin{equation}\label{eq:lambda}
\lambda_{n}(\mu):=
\min \left\{ 
\begin{array}{l}
(1-\mu) \underline{\mathbb{P}}_{n}(x)-\mu \overline{\mathbb{P}}_{n}(\neg x) ,\\
(1-\mu) \overline{\mathbb{P}}_{n}(x)-\mu \underline{\mathbb{P}}_{n}(\neg x)
\end{array}
\right\}\,,
\end{equation}
where the lower and upper probabilities in the above expression are those associated with the bounds in the CS specification for $X=\top$ and the other values are obtained by the conjugacy relation $\underline{\mathbb{P}}(x)=1-\overline{\mathbb{P}}(\neg x)$.
\item[(iii)] For any node $n$ normalized for vtree node $v$, for $z\in \mathbb{R}$:
\begin{equation}\label{eq:sig}
\underline{\sigma}_n(z) = \begin{cases} 
\overline{\mathbb{P}}_{n}(\bm{e}_v) & \text{ if } z < 0 \\
\underline{\mathbb{P}}_{n}(\bm{e}_v) & \text{ otherwise}\,,
\end{cases}\end{equation}
for $n\neq \bot$, while if $n=\bot$ we set $\underline{\sigma}_n(z)= 0$ for any $z\in \mathbb{R}$.
\end{enumerate}
We are ready to define Algorithm \ref{alg:lowercon}.
\begin{algorithm}[htp!]
    \caption{Lower conditional probability}
    \label{alg:lowercon}
    \begin{algorithmic} 
        \STATE {\bf input:} CSDD, $\mu$, $X=x$, {$\bm{e}$ }
        \FOR{$n \leftarrow 1,\ldots,N$}
        		\STATE{ $\underline{\pi}(n) \leftarrow 0$}
        		\STATE{$v\leftarrow $  vtree node that $n$ is normalized for}
       		 \IF{node $n$ is terminal, $n \neq \bot$}
       		 	\STATE{$\underline{\pi}(n) \leftarrow \Lambda_{n}(\mu)$  as in Eq. \eqref{eq:gamma}}
      		  \ELSE
        			\STATE{ $({(p_i,s_i)}_{i=1}^k,\mathbb{K}_n(P)) \leftarrow n$ (decision node)}
        			\IF{$X$ occurs in $v$}	
				\IF{$X$ occurs in $v^l$}
        					\STATE{$u_i \leftarrow p_i$ and $w_i \leftarrow s_i$ for $1\leq i \leq k$}
        				\ELSIF{$X$ occurs in $v^r$}
        					\STATE{$u_i \leftarrow s_i$ and $w_i \leftarrow p_i$ for $1\leq i \leq k$}
        				\ENDIF
			\STATE{ $\underline{\pi}(n) \leftarrow \min_{[\theta_1,\dots,\theta_k] \in \mathbb{K}_n(P)} \sum_{i=1}^k \underline{\pi}(u_i) \cdot \underline{\sigma}_{w_i}(\underline{\pi}(u_i)) \cdot 				\theta_i$} 
            \STATE{with $\underline{\sigma} $ as in Eq. \eqref{eq:sig}} 
        			\ELSE
				\STATE{$\underline{\pi}(n) \leftarrow \underline{\rho}_n(\mu)$}
        			\ENDIF    
         	\ENDIF
        \ENDFOR
        \STATE {\bf output: $\mathrm{sign}[\underline{\mathbb{P}}(x|\bm{e})-\mu] \leftarrow  \mathrm{sign} [\underline{\pi}(N)]$ }
    \end{algorithmic}
\end{algorithm}

The following result proves the correctness of Algorithm \ref{alg:lowercon} for singly connected CSDDs.

\begin{theorem}\label{theo:conditional}
Consider a singly connected CSDD and a node $n\neq \bot$ normalized for vtree node $v$, whose variables are $\bo{X}$. For any instantiation $x$ of a single variable $X\in \bo{X}$ and any coherent evidence $\bo{e}$ over some or all of the remaining variables, 
\begin{equation}
\underline{\pi}(n) = \min_{\mathbb{P}(\bm{X}) \in \mathbb{K}^n(\bm{X})}  \left[ (1-\mu) \mathbb{P}(x,\bm{e})-\mu \mathbb{P}(\neg x,\bm{e}) \right]\,.
\end{equation}
where $\underline{\pi}(n)$ is the message of node $n$ in Algorithm \ref{alg:lowercon}.
\end{theorem}

Observe that, both for terminal and decision nodes whose variables do not contain the queried variable $X$, the value  $\underline{\pi}(n)$ does not really matter, meaning that it does not affect the computation of the messages of the nodes processed after them. Indeed, consider a node $n'$ (terminal or decision) appearing as prime or sub in a decision node $n$, and assume $X$ occurs in $n$ but not in $n'$. Then the message $\underline{\pi}(n')$ will not contribute to $\underline{\pi}(n)$, but $\underline{\sigma}_{n'}(\underline{\pi}(n''))$ will, instead, where $n''$ is the node that, together with $n'$, forms an element of $n$. An implementation of Algorithm \ref{alg:lowercon} might therefore simply set $\underline{\pi}(n)=0$ for each node $n$ in which the queried variable does not occur, in order to avoid useless computations.

The procedure described by Algorithm \ref{alg:lowercon} requires the solution of a number of linear programming tasks, whose feasible regions are the CSs associated with the CSDD, equal to the number of decision nodes. The computation of the coefficients of the objective function in these tasks requires a call of Algorithm \ref{alg:lowerlik} for each optimization variable to compute the quantities in Equation \eqref{eq:sig}. Note also that, for each decision node $n=(\{(p_i,s_i)\}_{i=1}^k, \mathbb{K}_n(P))$  the optimization in the recursive call is performed before the one in Equation \eqref{eq:sig}. As discussed before, by iterated calls of Algorithm \ref{alg:lowercon}, we can therefore compute lower conditional queries in polynomial time in singly connected CSDDs.

\begin{example}\label{ex:conditional} Let us demonstrate how Algorithm \ref{alg:lowercon} works in practice by considering the same CSDD, with the same training data, as in the Example \ref{ex:marginal}. Consider the query $X_1=\top$ given evidence  $(X_2=\bot, X_3=\bot, X_4= \top)$. Take a generic $\mu \in [0,1]$. As the queried variable is the left-most variable in the variables ordering induced by the vtree in Figure \ref{fig:vtree1}, the output of Algorithm \ref{alg:lowercon} is the result of the following minimization:
\begin{equation}\label{eq:rootcond}
\min_{\substack{\theta_1\in \left[ \frac{31}{101},\frac{32}{101} \right]\\ \theta_2  \in \left[ \frac{52}{101},\frac{53}{101}\right]}}\lo{\pi}(24)\cdot\lo{\sigma}_{25}(\lo{\pi}(24))\cdot \theta_1 + \lo{\pi}(26)\cdot\lo{\sigma}_{27}(\lo{\pi}(26))\cdot \theta_2 + \lo{\pi}(28)\cdot\lo{\sigma}_{29}(\lo{\pi}(28))\cdot( 1- \theta_1 - \theta_2)\,.
\end{equation}
Computing $\lo{\pi}(24)$ requires no minimization because of the sharp parameters on the arcs of node $24$ and its value is $\lo{\pi}(0)\cdot\lo{\sigma}_1(\lo{\pi}(0))\cdot 1$. As node $0$ is a terminal node containing the queried variable, $\lo{\pi}(0)=\lambda_0(\mu)$. The latter quantity is equal to $-\mu$ because the query $X_1=\top$ does not agree with node $0$ whose literal is $\neg X_1$. Since $\lo{\pi}(0)<0$, $\lo{\sigma}_1(\lo{\pi}(0))= \up{\mathbb{P}}_1(X_2=\bot) =1$. Hence, $\lo{\pi}(24)= -\mu<0$, and $\lo{\sigma}_{25}(\lo{\pi}(24))= \up{\mathbb{P}}_{25}(X_3=\bot, X_4=\top)= \frac{13}{32}$.
The value of $\lo{\pi}(26)$ is the result of the following minimization:
\begin{equation}
\min_{\theta_4  \in \left[ \frac{39}{53},\frac{40}{53}\right]} \lo{\pi}(8)\cdot\lo{\sigma}_9(\lo{\pi}(8))\cdot \theta_4 + \lo{\pi}(10)\cdot\lo{\sigma}_{11}(\lo{\pi}(10))\cdot (1-\theta_4)\,.
\end{equation}
Both node $8$ and node $10$ contain the queried variable, hence $\lo{\pi}(8)=\lambda_8(\mu)=(1-\mu)$ and $\lo{\pi}(10)=\lambda_{10}(\mu)=-\mu$. Accordingly to the signs of the latter, $\lo{\sigma}_9(\lo{\pi}(8))= \lo{\mathbb{P}}_9(X_2=\bot)=1$ and $\lo{\sigma}_{11}(\lo{\pi}(10))= \up{\mathbb{P}}_{11}(X_2=\bot)=0$. Hence, $\lo{\pi}(26)= (1-\mu)\cdot \frac{39}{53}>0$. Moreover, $\lo{\sigma}_{27}(\lo{\pi}(26))$ is equal to $\lo{\mathbb{P}}_{27}(X_3=\bot, X_4=\top)$ and hence corresponds to:
\begin{align*}
\min_{\theta_5  \in \left[ \frac{8}{53},\frac{9}{53}\right]} \lo{\mathbb{P}}_{12}(X_3=\bot)\cdot \lo{\mathbb{P}}_{13}(X_4=\top)\cdot \theta_5 + \lo{\mathbb{P}}_{14}(X_3=\bot)\cdot \lo{\mathbb{P}}_{15}(X_4=\top)\cdot (1-\theta_5)\\
= \min_{\theta_5  \in \left[ \frac{8}{53},\frac{9}{53}\right]} \frac{33}{45}\cdot (1-\theta_5)
 = \frac{33}{45}\cdot \frac{44}{53}
 = \frac{484}{795}\,.
\end{align*}
One can easily verify that $\lo{\pi}(28)=0$. Thus, the minimization of Equation \eqref{eq:rootcond} rewrites as the following linear programming task:
\begin{equation}
\min_{\substack{\theta_1\in \left[ \frac{31}{101},\frac{32}{101} \right]\\ \theta_2  \in \left[ \frac{52}{101},\frac{53}{101}\right]}} -\mu \cdot \frac{13}{32}\cdot \theta_1 + (1-\mu)\cdot \frac{39}{53}\cdot \frac{484}{795} \cdot \theta_2\,,
\end{equation}
whose optimum is a numerical zero for $\mu \simeq 0.657$.
\end{example}

The assumption of singly connected topology is crucial for the proof of Theorem \ref{theo:conditional}. Yet, nothing prevents us from applying Algorithm \ref{alg:lowercon} to a multiply connected CSDD. Considered the last iteration of the algorithm leading to the value of $\mu$ for which the output of Algorithm \ref{alg:lowercon} is a numerical zero. 
The CSs associated with nodes of multiplicity higher than one have been used more than once as the feasible region of a linear programming tasks during the recursive calls of the algorithm. If the optima of those linear programming tasks corresponds to different extreme points of the same CS, we might have that an outer approximation has been introduced, i.e., the estimate of the lower (upper) probability returned by the algorithm is smaller (greater) than the exact one. Vice versa, if this is not the case, we might conclude that the algorithm returned an exact inference. To check this, we only need to store the extreme points of the CSs leading to the optima of the different linear programming tasks executed by the algorithm. In other words, no additional computational costs are required to decide whether or not the output of the algorithm is exact. Moreover, if an approximation has been introduced, a simple brute-force approach to the computation of the exact solution consists in running the same inferential task in the PSDDs compatible with the input CSDD and such that: (i) the PMFs of the nodes with multiplicity one and of the nodes with multiplicity more than one in case all the linear programming tasks have the same optimum are just the extreme points of the CSs that led to the optimum; (ii) the PMFs for the other nodes are any possible extreme points of the CSs, each with its multiplicity. This represents a brute-force algorithm involving a number of PSDD inference tasks exponential in the number of credal sets such as in (ii). These ideas are clarified by the following example.

\begin{example}
Consider a CSDD over the PSDD structure in Figure \ref{fig:multi}.whose CSs are all \emph{precise} (i.e., made of a single PMF) apart from specifications for each node except for node $3$ for which we assume a CS induced by the constraint $\theta_1 \in [l,u]$. Consider the conditional query $X_1=\top$ given evidence $X_3=\top$. For a given $\mu \in $ $ ]0,1[$, it is straightforward to verify that the messages $\underline{\pi}(n)$ of terminal nodes $n\in \{0,1,2,3,4,7,10,12,14\}$ are all equal to zero, while $\underline{\pi}(5)=\underline{\pi}(9)=1-\mu$ and $\underline{\pi}(6)=\underline{\pi}(8)=-\mu$. Consider now the decision nodes $11$ and $13$, sharing node $4$. We have:
\begin{equation}\label{eq:theta1a}
\underline{\pi}(11)=\lo{\pi}(5)\cdot \lo{\sigma}_4(\lo{\pi}(5))\cdot 1 + \lo{\pi}(6)\cdot \lo{\sigma}_7(\lo{\pi}(6))\cdot 0
 = (1-\mu)\cdot \lo{\mathbb{P}}_4(X_3=\top)\,,
\end{equation}
and
\begin{equation}\label{eq:theta1b}
\underline{\pi}(13)=\lo{\pi}(8)\cdot \lo{\sigma}_4(\lo{\pi}(8))\cdot 1 + \lo{\pi}(9)\cdot \lo{\sigma}_10(\lo{\pi}(9))\cdot 0 = -\mu\cdot \up{\mathbb{P}}_4(X_3=\top)\,.
\end{equation}
The two optimizations in Equations \eqref{eq:theta1a} and \eqref{eq:theta1b} with respect to $\theta_1$ give divergent values, i.e., $\theta_1=l$ in the first case and $\theta_1=u$ in the second. This is not consistent with the definition of strong extension in Section \ref{sec:csdd} and it would lead to an approximate value of the lower probability smaller than the exact one because of fewer constraints.
\end{example}

\section{Robustness of MAP inference in PSSDs}\label{sec:map}
CSDD can be also intended as tool for sensitivity analysis in PSDDs. Here we show how to evaluate the \emph{robustness} of a MAP inference in a PSDD. Let us first apply Algorithm \ref{alg:map} to a PSDD rooted at $r$ with evidence $\bo{e}$. We might ask ourselves whether or not the resulting configuration is sensitive to variations in the PSDD parameters. In order to do so, we also consider a CSDD the PSDD is consistent with. If all the PSDDs consistent with this CSDD have the same optimal configuration, and hence this is equal to the one obtained in the original PSDD, we say that the MAP inference is \emph{robust}. The following definition formalizes this idea.

\begin{definition} Given a PSDD $r$ over variables $\{\bo{X}, \bo{E}\}$ - with $\bo{X}$ and $\bo{E}$ disjoint - and an evidence $\bo{e}$ over variables $\bo{E}$, $\bo{x}^* := \arg \max_{\bm{x} \in \bm{\mathcal{X}}} \mathbb{P}_r(\bo{x},\bm{e})$ is robust with respect to a CSDD with which $r$ is consistent if:
\begin{equation}\label{eq:defrob}
\mathop{\max}_{\bo{x}\neq \bo{x}^*} \mathop{\max}_{\mathbb{P} \in \mathbb{K}^r} \frac{\mathbb{P}(\bo{x}, \bo{e})}{\mathbb{P}(\bo{x}^*, \bo{e})} < 1\,.
\end{equation}
\end{definition}
If $(\bo{x}^*,\bo{e}$) is inconsistent with $r$, we say that the inference is not robust by definition and give to the maximum in Equation \eqref{eq:defrob} a reference value one.

Algorithm \ref{alg:mapcsdd} is a subroutine used to decide the robustness of a MAP instance. It takes as input a CSDD rooted at $r$ over variables $\{\bo{X}, \bo{E}\}$ - with $\bo{X}$ and $\bo{E}$ disjoint - and an evidence $\bo{e}$ over variables $\bo{E}$, and computes $\mathop{\max}_{\bo{x}\in \bm{\mathcal{X}}} \overline{\mathbb{P}}_r(\bo{x},\bm{e})$. 

\begin{algorithm}[htp]
    \caption{"Credal MAP"}
    \label{alg:mapcsdd}
    \begin{algorithmic} 
     \STATE {\bf input:} CSDD $r$, evidence $\mathbf{e}$
     \FOR{$n \leftarrow 1,\ldots,N$}
        \STATE{ $M(n) \leftarrow 0$}
        \IF{node $n$ is terminal}
             \STATE{$v\leftarrow $ leaf vtree node that $n$ is normalized for}
              \IF{$var(v)\in \bo{X}$}
                  \IF{$n \in \{ X, \neg X \}$}
                       \STATE{$M(n)\leftarrow 1$}
                 \ELSIF{$n= (X, [l,u])$}
                       \STATE{$M(n)\leftarrow \max \{ u, 1-l\}$}
                 \ENDIF      
            \ELSE
                 \STATE{$M(n)\leftarrow \overline{P}_n(\bo{e}_v) $}        
           \ENDIF
        \ELSE
            \STATE{ $({(p_i,s_i)}_{i=1}^k,\mathbb{K}_n(P)) \leftarrow n$ (decision node)} 
            \STATE{ $M(n) \leftarrow  \max_{1\leq i \leq k} \up{\theta_i} \cdot M(p_i) \cdot M(s_i)$
            with  $\up{\theta_i} :=\max_{\mathbb{K}_{n}} \theta_i$}
        \ENDIF
        \ENDFOR
        \STATE {\bf output:} $M(N)$ 
    \end{algorithmic}
\end{algorithm}

The following theorem gives a semantics for the output of Algorithm \ref{alg:mapcsdd}.

\begin{theorem}\label{theo:credalmap}  Consider a CSDD and a node $n\neq \bot$ normalized for vtree node $v$ whose variables are $\{\bo{X}, \bo{E}\}$,  with $\bo{X}$ and $\bo{E}$ disjoint. Let $\bo{e}$ be a total evidence over variables $\bo{E}$. Then:
\begin{equation}
M(n) = \mathop{\max}_{\bo{x}\in \bm{\mathcal{X}} } \overline{\mathbb{P}}_n(\bo{x},\bm{e})\,.
\end{equation}
\end{theorem}

Algorithm \ref{alg:robustness}  is used to decide the robustness of a MAP inference $\bo{x}^* \in \arg \mathop{\max}_{\bm{x}\in \bm{\mathcal{X}}}\mathbb{P}_r(\bm{x},\bm{e})$ in PSDD $r$ in the following way. Following a topological order, each node $n$ is processed and gives message $V(n)$, which is a relaxed version of the left-hand side of Equation \ref{eq:defrob}, in which we do not require the configurations $\bm{x}\in \bm{\mathcal{X}}$ to be distinct from $\bm{x}^*$ (with the adequate restrictions to $n$'s variables). Observe that the message of decision nodes $n$ not realized by $(\bm{x}^*_v,\bm{e}_v)$ is $0$. In fact, this value does not matter: the contribution of such nodes will be taken into account - as Credal-MAP message - when processing the first higher decision node consistent with (the adequate restriction of) $(x^*,\bm{e})$.

Because of the previously relaxed constraint, the message of the root $V(r)$ is greater or equal than $1$. If $V(r) > 1$, we can conclude that $\bo{x}^*$ is not robust. If $V(r)=1$, we need to re-take into account the constraint. In order to do so, we observe:

\begin{itemize}
\item if $\bo{x}^*$ is the only configuration realizing the maximum, we can state its robustness;
\item if $\bo{x}^*$ is between several configurations realizing the maximum, we can say that it is \textit{weakly} robust;
\item if $\bo{x}^*$ does not realize the maximum, we conclude that it is not robust.
\end{itemize}

Note that Equation \ref{eq:defrob} holds if and only if the first situation occurs.







\begin{algorithm}[htp]
    \caption{Robustness}
    \label{alg:robustness}
    \begin{algorithmic} 
        \STATE {\bf input:} CSDD $r$, evidence $\mathbf{e}$, $\bo{x}^* \in arg \mathop{\max}_{\bo{x}\in val(\bo{X})}\mathbb{P}_r(\bo{x},\bm{e})$
        \FOR{$n \leftarrow 1,\ldots,N$}
        \STATE{ $V(n) \leftarrow 1$}
        \STATE{$v\leftarrow $ leaf vtree node that $n$ is normalized for}

\IF{node $n$ is terminal}
        \IF{$n= (X : [l,u])$}
             \STATE{$V(n) \leftarrow \max\{1, \tfrac{1-l}{l}\}$ when $\bo{x}_v^*= \top$} 
             \STATE{$V(n) \leftarrow  \max\{1, \tfrac{u}{1-u}\}$ when $\bo{x}_v^*= \bot$}
        \ENDIF
    

\ELSE
        \STATE{ $(\{(p_i,s_i)\}_{i=1}^k,\mathbb{K}_n(P)) \leftarrow n$ (decision node)} 
        \IF{$x^*_v\bo{e}_v \models \interp{n}$}
             \STATE{$V(n)\leftarrow \max \{ V(p_j)\cdot V(s_j), \max_{1\leq i \leq k, i \neq j} U_{i,j}\} $} // $j$ is the unique index such that $\bo{x}_{v,l}^*\bo{e}_{v,l}\models \interp{p_j}$  and 
             $U_{i,j} := \max_{\mathbb{K}_n}\tfrac{\theta_i \cdot M(p_i)\cdot M(s_i)}{\theta_j \cdot \underline{P}_{p_j}(\bo{x}_{v,l}^*,\bo{e}_{v,l})\cdot \underline{P}_{s_j}(\bo{x}_{v,r}^*,\bo{e}_{v,r})}$
       \ENDIF
\ENDIF
\ENDFOR
\STATE {\bf output:} $V(N)$
    \end{algorithmic}
\end{algorithm}

The following theorem states the correctness of Algorithm \ref{alg:robustness} for singly connected CSDDs.

\begin{theorem}\label{theo:rob}  Let $r$ be a singly connected CSDD over variables $\{\bo{X}, \bo{E}\}$, with $\bo{X}$ and $\bo{E}$ disjoint. Consider an evidence $\bo{e}$ over variables $\bo{E}$ and an instance $\bo{x}^* \in \arg \mathop{\max}_{\bo{x}\in \bm{\mathcal{X}}}\mathbb{P}_r(\bo{x},\bm{e})$ obtained by applying Algorithm \ref{alg:map} to a consistent PSDD. For each node $n\neq \bot$ in $r$ normalized for vtree node $v$:
\begin{equation}
V(n) = \mathop{\max}_{\bo{x}_v\in val(\bo{X}_v)} \mathop{\max}_{\mathbb{P} \in \mathbb{K}^n} \frac{\mathbb{P}(\bo{x}_v, \bo{e}_v)}{\mathbb{P}(\bo{x}_v^*, \bo{e}_v)}\,.
\end{equation}
\end{theorem}

The motivations for which we are not in measure to state the theorem for general CSDDs are analogous to the ones for conditional inference. In the induction step of the previous proof, in the case of $i\neq j$, we perform a maximization on the numerator and a minimization on the denominator, this being possible because nodes on the numerator and nodes on the denominator have distinct CSs. Nevertheless, this does not prevent the algorithm from selecting several distinct optimal sub-configurations in the case of a multiple node possibly shared by $p_i$ and $p_j$, or $s_i$ and $s_j$, when the CSDD is multiply connected. Thus, exactly as in the case of conditional queries, we obtain an outer approximation meaning that the output of Algorithm \ref{alg:robustness} might be greater than the left-hand side of Equation \eqref{eq:defrob}. In other words, for multiply connected models, if the algorithm says that the configuration is robust we are certain, while it might be the case that the algorithm says that the configuration is not robust, while this is not the case. This might be therefore intended as a conservative approximation. Finally, exactly as in the conditional case, we might decide whether or not the algorithm returned an approximation by simply inspecting the extreme points of the CSs with multiplicity higher than one leading to the optima of the linear programs solved during the execution of the algorithm and, in case of approximation, run a brute-force algorithm exponential in the number of CSs for which different tasks gave different optimal extreme points.
\section{Experiments}\label{sec:experiments}
As a first application of the algorithms derived in the
previous section, we consider a simple machine learning task
involving logical constraints over the model variables. The
problem consists in the identification of the digit depicted
by a seven-segment display (Figure \ref{fig:seven_segments}),
whose segments might occasionally fail to turn on. More
specifically, given an input digit to be displayed, the
control unit activates the corresponding set of segments in
the display; each segment can however fail to be switched on
independently with an identical probability.
We note that while this scenario is relatively simple, it
can easily be extended to more complex and realistic scenarios
involving a large number of components/devices, whose
interdependence is described as a logical function, and 
whose probability of failures are interconnected in a
complicated way.

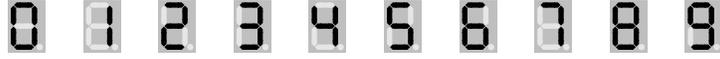
\begin{figure}[h]
\centering
\begin{circuitikz}[scale=1.]
\ctikzset{seven seg/width=0.2, seven seg/thickness=2pt}
\ctikzset{seven seg/color on=black, seven seg/color off=black!10}
\foreach \i in {0,...,9} \path (\i,-1.5) node[seven segment val=\i dot off box off, fill=gray!50!white]{};
\end{circuitikz}
\caption{Digits represented by a seven-segment display\label{fig:seven_segments}}
\end{figure}

Our setup can be described by fourteen Boolean variables: say that $\bm{X}:=(X_1,\ldots,X_7)$ are the  \emph{hidden} states of the segments as decided by the control unit, and $\bm{O}:=(O_1,\ldots,O_7)$ are the \emph{observable} states of the segments as depicted in the display. Let us also assume that the true state of these Boolean variables corresponds to the segment on.

We create synthetic data as follows. Given digit $j$, the corresponding configuration of $\bm{X}$ is provided by the formula $\delta_j(\bm{X})$ as in Table \ref{tab:digits}. Then, for each $i=1,\ldots,7$, if $X_i$ is false, we also set $O_i$ false, while if $X_i$ is true, $O_i$ might be false with a given failure probability $p_f$. Such mechanism obeys the formula:
\begin{equation}\label{eq:digits_formula}
\phi := \wedge_{i=1}^7 (O_i \rightarrow X_i) \wedge \left( \vee_{j=0}^{9} \delta_j(X_1,\ldots,X_7) \right)\,.
\end{equation}

\begin{table}[htp!]
\centering
\begin{tabular}{lc}
\hline
$j$&$\delta_j(\bm{X})$\\
\hline
$0$ & $X_1 \wedge X_2 \wedge X_3 \wedge X_4 \wedge X_5 \wedge X_6 \wedge \neg X_7$\\
$1$ & $\neg X_1 \wedge X_2 \wedge X_3 \wedge \neg X_4 \wedge \neg X_5 \wedge \neg X_6 \wedge \neg X_7$\\
$2$ & $X_1 \wedge X_2 \wedge \neg X_3 \wedge X_4 \wedge X_5 \wedge \neg X_6 \wedge X_7$\\
$3$ & $X_1 \wedge X_2 \wedge X_3 \wedge X_4 \wedge \neg X_5 \wedge \neg X_6 \wedge X_7$\\
$4$ & $\neg X_1 \wedge X_2 \wedge X_3 \wedge \neg X_4 \wedge \neg X_5 \wedge X_6 \wedge X_7$\\
$5$ & $X_1 \wedge \neg X_2 \wedge X_3 \wedge X_4 \wedge \neg X_5 \wedge X_6 \wedge X_7$\\
$6$ & $X_1 \wedge \neg X_2 \wedge X_3 \wedge X_4 \wedge X_5 \wedge X_6 \wedge X_7$\\
$7$ & $X_1 \wedge X_2 \wedge X_3 \wedge \neg X_4 \wedge \neg X_5 \wedge \neg X_6 \wedge \neg X_7$\\
$8$ & $X_1 \wedge X_2 \wedge X_3 \wedge X_4 \wedge X_5 \wedge X_6 \wedge \neg X_7$\\
$9$ & $X_1 \wedge X_2 \wedge X_3 \wedge X_4 \wedge X_5 \wedge X_6 \wedge X_7$\\
$0$ & $X_1 \wedge X_2 \wedge X_3 \wedge X_4 \wedge \neg X_5 \wedge X_6 \wedge X_7$\\
\hline
\end{tabular}
\caption{Digits configuration as disjunctive formulae\label{tab:digits}}
\end{table}

Given formula $\phi$ in Equation \eqref{eq:digits_formula}, we use the algorithm proposed in \cite{choi2013} to build an SDD $\alpha$ normalized for a vtree such that, for each $i=1,\ldots,7$, the pair $(X_i,O_i)$ corresponds to a pair of leaves with the same parent and with a so-called \emph{balanced} shape. The resulting SDD has a multiply connected structure, 128 nodes (82 of them decision nodes) and maximum number of elements for decision node equal to eight.

Given a training data set $\mathcal{D}$ of size $d$, generated according to the above described procedure, we can obtain from $\alpha$ a PSDD or a CSDD. In the first case we use a Bayesian procedure, with Perks' prior and equivalent sample size $s=1$, to learn PMFs associated with the decision nodes and the non-bot terminal nodes. In the second case, IDM with the same equivalent sample size is used to learn the CSs. 

As a rival setup we consider a \emph{hidden Markov model} (HMM) whose hidden variables are those in $\bm{X}$, while the observations are those in $\bm{O}$. The model is trained from the same data set $\mathcal{D}$ and with he same prior as the PSDD. A credal extension of HMMs, perfectly analogous to the one we presented here for PSDDs, have been proposed in \cite{maua2016hidden}. Thus, we can also quantify the HMM parameters as CSs obtained by IDM with the same equivalent sample size. We refer to this model as IHMM, while HMM is its precise counterpart.

Given a test instance $(\bm{x}',\bm{o}')$, generated by the same mechanism discussed for the training set, we therefore have four different models to perform reasoning. As a first task, we predict, given the observation $\bm{o}'$, the most probable configuration of $X_i'$ for each $i=1,\ldots,7$. In the PSDD, this is prediction is driven by the conditional inference $P(X_i'=1|\bm{o}')$. The same can be done with the HMM by the classical \emph{filtering} algorithm (we create a different HMM for each $i$ such that $X_i$ and $O_i$ are always the last elements of the sequence). For the CSDD, Algorithm \ref{alg:lowercon} is used instead to compute posterior intervals $[\underline{P}(X_i'=1|\bm{o}'),\overline{P}(X_i'=1|\bm{o}')]$, while the same task can be solved in polynomial time also in IHMMs by the (credal) filtering algorithm proposed in \cite{maua2016hidden}. With 0/1 losses, the rule to decide whether or not the segment $X_i'$ is on according to a PSDD or HMM is simply whether or not the probability of the true state is larger than half, the segment being off otherwise. For CSDDs and IHMMs, we say that the segment is certainly on, if the lower conditional probability is more than half, and certainly off if the upper probability is less than half. If none of the two above cases is satisfied, we say that we are in a condition of \emph{indecision} between the two options. This is an example of so-called \emph{credal classifier} \cite{zaffalon2002naive}, which suspends the judgement about the actual state of the segment when the available information is not sufficient to take a determinate decision.

In summary, given $\bm{o}'$, we classify each segment
separately by using: (i) PSDDs and HMMs as standard
classifiers, whose performance is described by the accuracy,
i.e., the percentage of segments whose state was properly
recognized; (ii) CSDDs and IHMMs as credal classifiers, whose
performance is described by the $u_{80}$ utility-based
performance measure, which is commonly used to evaluate the
performance of a credal classifiers as it more properly balances the quality of the prediction and the lack of informativeness associated to indeterminate classifications and it is considered a proper measure to compare the performance of credal classifiers against the accuracy of a standard classifier \cite{zaffalon2012evaluating}.

In our experiments we consider training sets of size $d\in\{10,15,20,50,100\}$ and test the four models trained with these data with a test set of size $d'=140$. Different failure probabilities $p_f \in \{0.05,0.1,0.2,0.3,0.4\}$ are also considered. 

The CSDD inference algorithms have been implemented by the authors  in Python together with the necessary data structures.\footnote{\url{https://github.com/alessandroantonucci/pycsdd}} The PySDD library was used to build the SDDs associated with a formula.\footnote{\url{https://github.com/wannesm/PySDD}} The PyPSDD library was used instead to validate the consistency between PSDDs and CSDDs.\footnote{\url{https://github.com/art-ai/pypsdd}}. The iHMM library was finally used instead for experiments with HMMs/IHMMs.\footnote{\url{https://github.com/denismaua/ihmm}}

\begin{figure}[htp!]
    \centering
    \resizebox{6cm}{6cm}{
        \begin{tikzpicture}
        \begin{axis}[title={d=10},xlabel={$p_f$},ylabel={accuracy}]
        \addplot[color=red,dashed] table [x=pf,y=psdd,col sep=comma] {acc1.csv};
        \addplot[color=red] table [x=pf,y=csdd,col sep=comma] {acc1.csv};
        \addplot[color=blue,dashed] table [x=pf,y=hmm,col sep=comma] {acc1.csv};
        \addplot[color=blue] table [x=pf,y=ihmm,col sep=comma] {acc1.csv};
        \legend{PSDD,CSDD,HMM,IHMM}
        \end{axis}
        \end{tikzpicture}}
    \resizebox{6cm}{6cm}{
    \begin{tikzpicture}
    \begin{axis}[title={d=15},xlabel={$p_f$},ylabel={accuracy}]
    \addplot[color=red,dashed] table [x=pf,y=psdd,col sep=comma] {acc2.csv};
    \addplot[color=red] table [x=pf,y=csdd,col sep=comma] {acc2.csv};
    \addplot[color=blue,dashed] table [x=pf,y=hmm,col sep=comma] {acc2.csv};
    \addplot[color=blue] table [x=pf,y=ihmm,col sep=comma] {acc2.csv};
    \legend{PSDD,CSDD,HMM,IHMM}
    \end{axis}
    \end{tikzpicture}}
    \resizebox{6cm}{6cm}{
    \begin{tikzpicture}
    \begin{axis}[title={d=20},xlabel={$p_f$},ylabel={accuracy}]
    \addplot[color=red,dashed] table [x=pf,y=psdd,col sep=comma] {acc3.csv};
    \addplot[color=red] table [x=pf,y=csdd,col sep=comma] {acc3.csv};
    \addplot[color=blue,dashed] table [x=pf,y=hmm,col sep=comma] {acc3.csv};
    \addplot[color=blue] table [x=pf,y=ihmm,col sep=comma] {acc3.csv};
    \legend{PSDD,CSDD,HMM,IHMM}
    \end{axis}
    \end{tikzpicture}}
    \resizebox{6cm}{6cm}{
    \begin{tikzpicture}
    \begin{axis}[title={d=50},xlabel={$p_f$},ylabel={accuracy}]
    \addplot[color=red,dashed] table [x=pf,y=psdd,col sep=comma] {acc4.csv};
    \addplot[color=red] table [x=pf,y=csdd,col sep=comma] {acc4.csv};
    \addplot[color=blue,dashed] table [x=pf,y=hmm,col sep=comma] {acc4.csv};
    \addplot[color=blue] table [x=pf,y=ihmm,col sep=comma] {acc4.csv};
    \legend{PSDD,CSDD,HMM,IHMM}
    \end{axis}
    \end{tikzpicture}}
    \resizebox{6cm}{6cm}{
    \begin{tikzpicture}
    \begin{axis}[title={d=100},xlabel={$p_f$},ylabel={accuracy}]
    \addplot[color=red,dashed] table [x=pf,y=psdd,col sep=comma] {acc5.csv};
    \addplot[color=red] table [x=pf,y=csdd,col sep=comma] {acc5.csv};
    \addplot[color=blue,dashed] table [x=pf,y=hmm,col sep=comma] {acc5.csv};
    \addplot[color=blue] table [x=pf,y=ihmm,col sep=comma] {acc5.csv};
    \legend{PSDD,CSDD,HMM,IHMM}
    \end{axis}
    \end{tikzpicture}}
\caption{Accuracies}\label{fig:acc}
\end{figure}
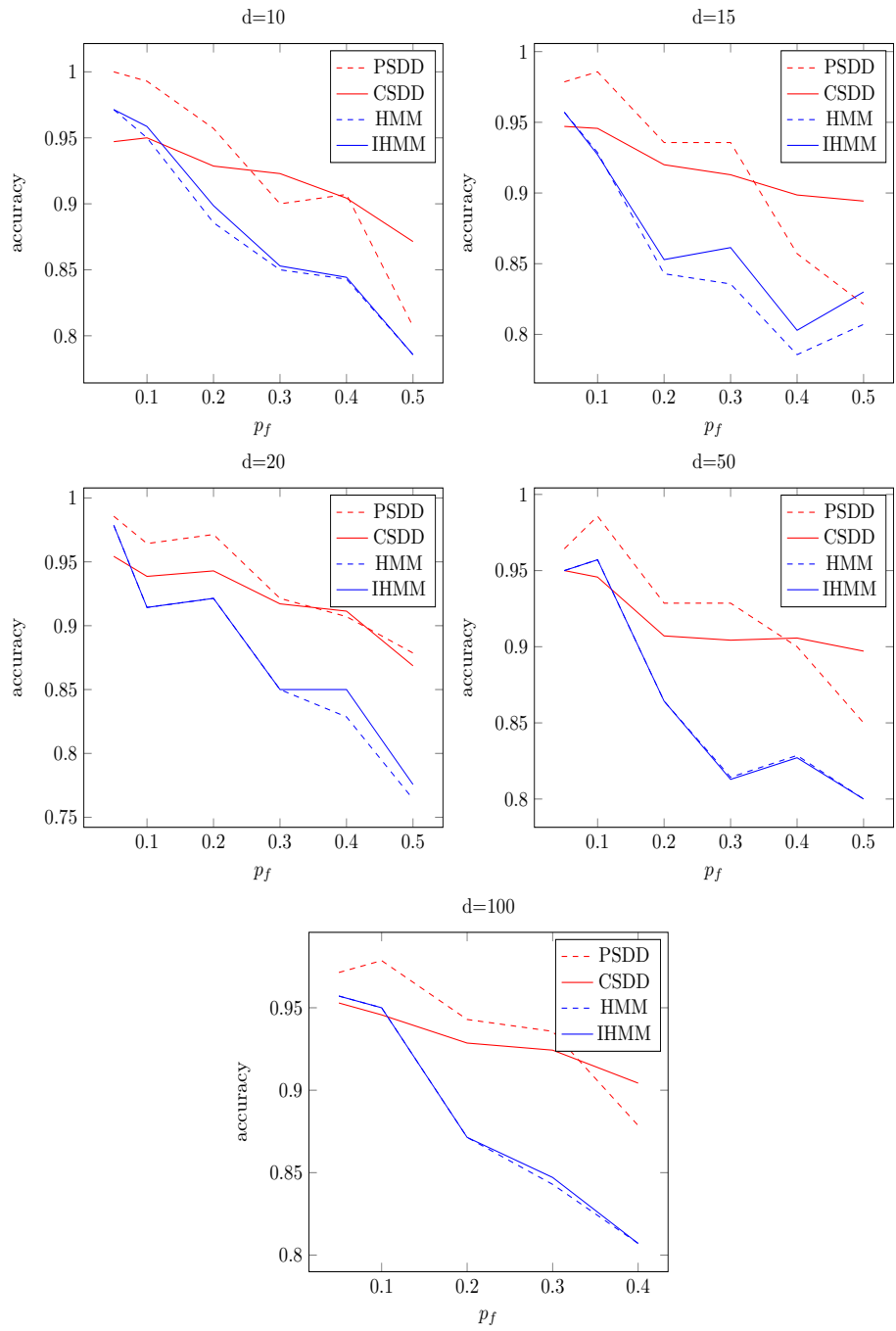

Figure \ref{fig:acc} depicts five plots showing the accuracies of the four different models as a function of $p_f$ for different training set sizes $d$. The behaviour is clear PSDDs/CSDDs models outperform HMMs/IHMMs most of the times, with the differences being typically narrower for low failure probabilities. This is expected and the gap between the two models should be intended as the effect of the additional information about the logical constraints in Equation \ref{eq:digits_formula}, that is not available to the HMMs/IHMMs. The smaller gap for low failure probabilities can be also explained by noticing that the emission term $P(O_j|X_j)$ involved in the parametrization of HMMs/IHMMs takes almost diagonal form for low failure probabilities and, in these cases, the observation of $O_j$ induces a high probability for the same state of $X_j$, thus making irrelevant the effect of the logical constraints. Moreover, we notice that the CSDD tends to outperform the PSDD for larger failure probabilities. This is also expected: increasing the noise level in the data promptly induces a degradation of the PSDD accuracy, while the CSDD is able to contain that effect by allowing for indeterminate classifications of some segments.

Credal classifier are typically used as preprocessing systems able to distinguish easy-to-classify instances for which the output of the standard method is considered sufficiently reliable, from the hard-to-classify ones, for which other dedicated and typically more demanding/expensive techniques should be invoked. Such a separation is naturally provided by the classifier, as it corresponds to the difference between the instances for which the output of the classifier is determinate and the other ones. A typical description of such discriminative power is the difference between the accuracies of the precise counterpart of a credal classifier on these two sets of instances. In Figure \ref{fig:discr}, we plot the so-called \emph{determinate} and \emph{indeterminate accuracies} of the PSDD, i.e., the accuracy of the PSDD on the instances (i.e., segments) for which the credal classifier was determinate or indeterminate. As expected, the CSDD is properly able to distinguish these two sets and keeps a level of accuracy very close to one even for high perturbation levels (the perturbation only affecting the determinacy, i.e., the percentage of determinate classifications).

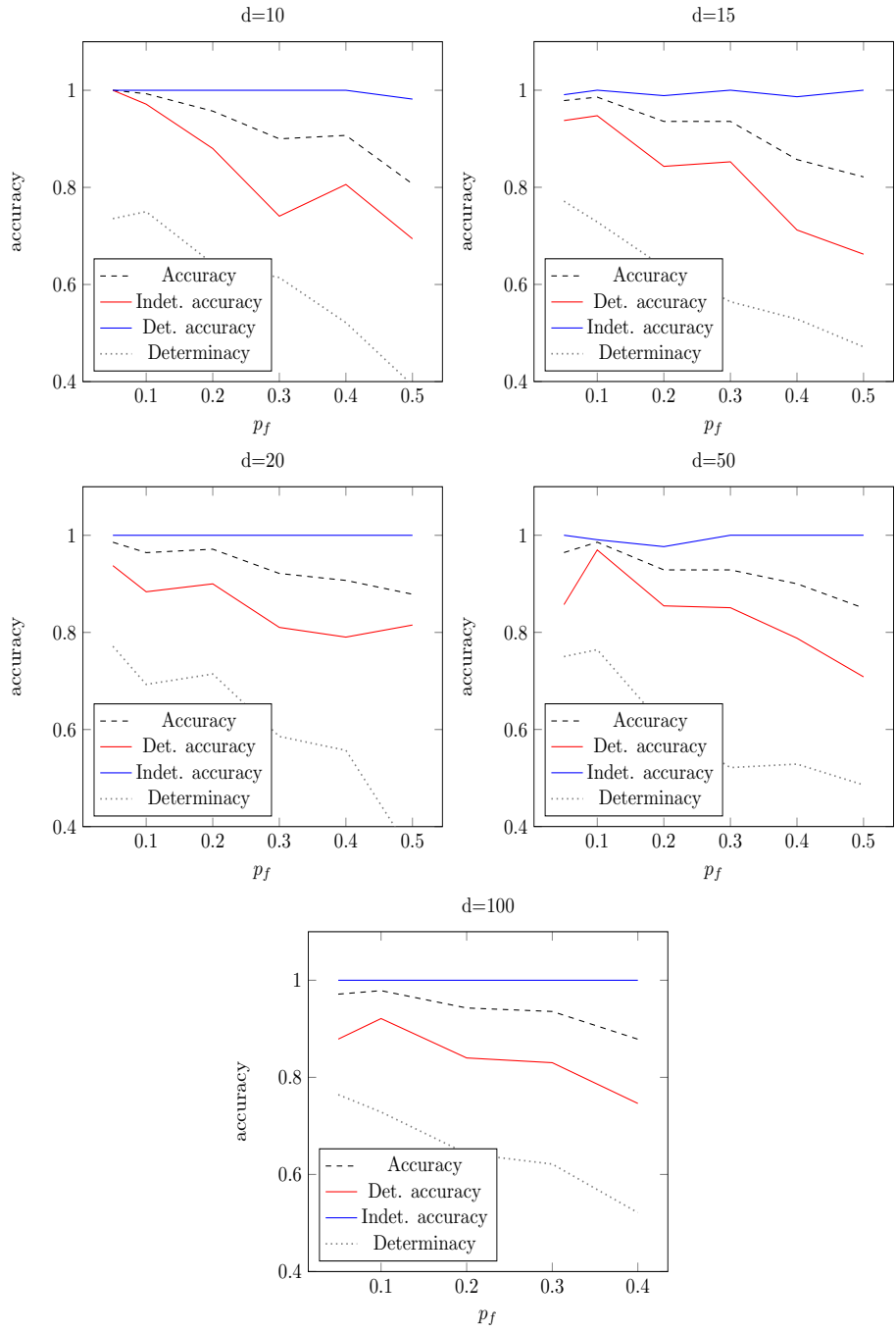
\begin{figure}[htp!]
    \centering
    \resizebox{6cm}{6cm}{
        \begin{tikzpicture}
        \begin{axis}[title={d=10},xlabel={$p_f$},ylabel={accuracy},legend pos={south west},ymin=0.4,ymax=1.1]
        \addplot[color=black,dashed] table [x=pf,y=acc,col sep=comma] {discr1.csv};
        \addplot[color=red] table [x=pf,y=det_acc,col sep=comma] {discr1.csv};
        \addplot[color=blue] table [x=pf,y=indet_acc,col sep=comma] {discr1.csv};
        \addplot[color=gray,dotted,thick] table [x=pf,y=det,col sep=comma] {discr1.csv};
        \legend{Accuracy,Indet. accuracy,Det. accuracy,Determinacy}
        \end{axis}
        \end{tikzpicture}}
    \resizebox{6cm}{6cm}{
    \begin{tikzpicture}
    \begin{axis}[title={d=15},xlabel={$p_f$},ylabel={accuracy},legend pos={south west},ymin=0.4,ymax=1.1]
    \addplot[color=black,dashed] table [x=pf,y=acc,col sep=comma] {discr2.csv};
    \addplot[color=red] table [x=pf,y=det_acc,col sep=comma] {discr2.csv};
    \addplot[color=blue] table [x=pf,y=indet_acc,col sep=comma] {discr2.csv};
    \addplot[color=gray,dotted,thick] table [x=pf,y=det,col sep=comma] {discr2.csv};
    \legend{Accuracy,Det. accuracy,Indet. accuracy,Determinacy}
    \end{axis}
    \end{tikzpicture}}
    \resizebox{6cm}{6cm}{
    \begin{tikzpicture}
    \begin{axis}[title={d=20},xlabel={$p_f$},ylabel={accuracy},legend pos={south west},ymin=0.4,ymax=1.1]
    \addplot[color=black,dashed] table [x=pf,y=acc,col sep=comma] {discr3.csv};
    \addplot[color=red] table [x=pf,y=det_acc,col sep=comma] {discr3.csv};
    \addplot[color=blue] table [x=pf,y=indet_acc,col sep=comma] {discr3.csv};
    \addplot[color=gray,dotted,thick] table [x=pf,y=det,col sep=comma] {discr3.csv};
    \legend{Accuracy,Det. accuracy,Indet. accuracy,Determinacy}
    \end{axis}
    \end{tikzpicture}}
    \resizebox{6cm}{6cm}{
    \begin{tikzpicture}
    \begin{axis}[title={d=50},xlabel={$p_f$},ylabel={accuracy},legend pos={south west},ymin=0.4,ymax=1.1]
    \addplot[color=black,dashed] table [x=pf,y=acc,col sep=comma] {discr4.csv};
    \addplot[color=red] table [x=pf,y=det_acc,col sep=comma] {discr4.csv};
    \addplot[color=blue] table [x=pf,y=indet_acc,col sep=comma] {discr4.csv};
    \addplot[color=gray,dotted,thick] table [x=pf,y=det,col sep=comma] {discr4.csv};
    \legend{Accuracy,Det. accuracy,Indet. accuracy,Determinacy}
    \end{axis}
    \end{tikzpicture}}
    \resizebox{6cm}{6cm}{
    \begin{tikzpicture}
    \begin{axis}[title={d=100},xlabel={$p_f$},ylabel={accuracy},legend pos={south west},ymin=0.4,ymax=1.1]
    \addplot[color=black,dashed] table [x=pf,y=acc,col sep=comma] {discr5.csv};
    \addplot[color=red] table [x=pf,y=det_acc,col sep=comma] {discr5.csv};
    \addplot[color=blue] table [x=pf,y=indet_acc,col sep=comma] {discr5.csv};
    \addplot[color=gray,dotted,thick] table [x=pf,y=det,col sep=comma] {discr5.csv};
    \legend{Accuracy,Det. accuracy,Indet. accuracy,Determinacy}
    \end{axis}
    \end{tikzpicture}}
    \caption{PSDD determinate vs. indeterminate accuracies}\label{fig:discr}
\end{figure}

Finally, for a validation of Algorithm \ref{alg:robustness}, we perform an analysis analogous to that in Figure \ref{fig:discr} but at the level of joint configuration of the hidden variables corresponding to a particular digit. In practice, we compute the MAP configuration of $\bm{X}=\bm{x}^*$ given $\bm{o}'$ in the PSDD and use Algorithm \ref{alg:robustness} to check whether or not the configuration was robust. The corresponding determinate and indeterminate, joint, accuracies are reported in Figure \ref{fig:discr2} only for $d\geq 20$ as for lower training set size the amount of detected digits is very low in both cases. As expected the behaviour is analogous to that in Figure \ref{fig:discr}.

\begin{figure}[htp!]
    \centering
    \resizebox{5cm}{5cm}{
        \begin{tikzpicture}
        \begin{axis}[title={d=20},xlabel={$p_f$},ylabel={accuracy},legend pos={south west},ymin=0.0,ymax=1.1]
        \addplot[color=black,dashed] table [x=pf,y=matching,col sep=comma] {joint3.csv};
        \addplot[color=red] table [x=pf,y=det_matching,col sep=comma] {joint3.csv};
        \addplot[color=blue] table [x=pf,y=indet_matching,col sep=comma] {joint3.csv};
        \addplot[color=gray,dotted,thick] table [x=pf,y=determinacy,col sep=comma] {joint3.csv};
        \legend{Accuracy,Det. accuracy,Indet. accuracy,Determinacy}
        \end{axis}
        \end{tikzpicture}}
    \resizebox{5cm}{5cm}{
        \begin{tikzpicture}
        \begin{axis}[title={d=100},xlabel={$p_f$},ylabel={accuracy},legend pos={south west},ymin=0.0,ymax=1.1]
        \addplot[color=black,dashed] table [x=pf,y=matching,col sep=comma] {joint5.csv};
        \addplot[color=red] table [x=pf,y=det_matching,col sep=comma] {joint5.csv};
        \addplot[color=blue] table [x=pf,y=indet_matching,col sep=comma] {joint5.csv};
        \addplot[color=gray,dotted,thick] table [x=pf,y=determinacy,col sep=comma] {joint5.csv};
        \legend{Accuracy,Det. accuracy,Indet. accuracy,Determinacy}
        \end{axis}
        \end{tikzpicture}}
    \caption{PSDD determinate vs. indeterminate joint accuracies}\label{fig:discr2}
\end{figure}
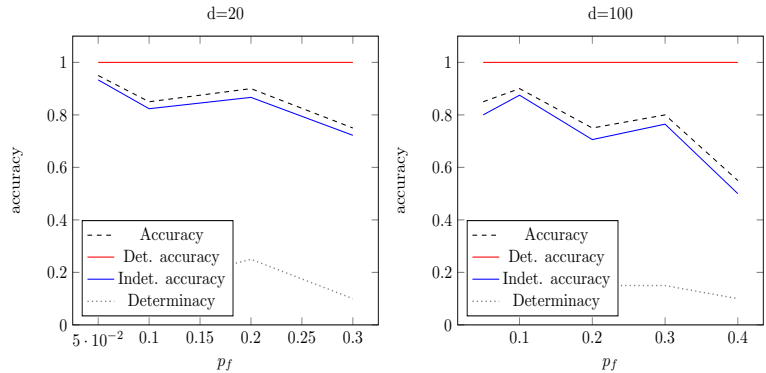
\section{Conclusions}\label{sec:conc}
We have introduced a new class of imprecise probabilistic graphical models based on a credal set extension of \emph{probabilistic sentential diagrams}. Three efficient algorithms for marginal, conditional and MAP queries are derived. The first algorithm is exact for any topology, while the second and the third might induce a conservative approximation in the multiply connected case. Yet, a fast procedure to test whether or not an approximation has been also derived. An empirical validation on a synthetic setup show that the credal extension allows to properly distinguish between easy-to-classify and hard-to-classify instances. Regarding the multiply connected case, whether or not for conditional queries and for the robustness of a MAP task, exact inferences can be efficiently computed remains an open question to be addressed as a future work.


\appendix
\section*{Proofs}
\begin{proof}[Proof of Theorem \ref{theo:map}] If $r$ is a terminal PSDD, it is easy to see the correctness of the algorithm. Suppose that $r$ is a decision node, $r = {(p_i,s_i,\theta_i)}_{i=1}^k$. For a given $\bo{x} \in val(\bo{X})$, $\bo{xe}$ is a total instantiation of its variables. By definition of PSDDs distribution, $\mathbb{P}_r(\bo{xe})= \sum_{i=1}^k \mathbb{P}_{p_i}(\bo{x}_l \bo{e}_l) \cdot \mathbb{P}_{s_i} (\bo{x}_l \bo{e}_r)\cdot \theta_i$. Now, remember that for each $\bo{x}$,  $\bo{x}_l\bo{e}_l$ realizes a unique prime, so this maximum is of the form $\mathbb{P}_{p_i}(\bo{x}_l \bo{e}_l) \cdot \mathbb{P}_{s_i} (\bo{x}_l \bo{e}_r)\cdot \theta_i$ for a unique $1\leq i \leq k$. Hence,

\begin{align*}
\mathop{\max}_{\bo{x}\in val(\bo{X})}\mathbb{P}_r(\bo{x},\bm{e}) & = \mathop{\max}_{1 \leq i \leq k} \mathop{\max}_{\bo{x}\in val(\bo{X})} \mathbb{P}_{p_i}(\bo{x}_l \bo{e}_l) \cdot \mathbb{P}_{s_i} (\bo{x}_r \bo{e}_r)\cdot \theta_i \\
& = \mathop{\max}_{1 \leq i \leq k} \theta_i \cdot [ \mathop{\max}_{\bo{x}_l\in val(\bo{X}_l)} \mathbb{P}_{p_i}(\bo{x}_l \bo{e}_l)] \cdot  [\mathop{\max}_{\bo{x}_r\in val(\bo{X}_r)} \mathbb{P}_{s_i}(\bo{x}_r \bo{e}_r)]\\
& = \mathop{\max}_{1 \leq i \leq k} \theta_i \cdot  MAP(p_i) \cdot  MAP(s_i)
\end{align*}
\end{proof}

\begin{proof}[Proof of Theorem \ref{theo:base}] 
Base case: Let $n$ be a terminal node normalized for leaf vtree node $v$. Let $X$ be the variable of leaf $v$ and $\bo{x}$ an instantiation of $X$. If $n=X$, on one hand $\top \models X$ and $\lo{\mathbb{P}}_X(\top)=1$, on the other hand $\bot \not\models X$ and  $\up{\mathbb{P}}_X(\bot)=0$. Similarly for $n = \neg X$. If $n = (X : [\alpha, \beta])$, $\lo{\mathbb{P}}_n(\top) = \alpha $ and $\lo{\mathbb{P}}_n(\bot) = 1-\beta$, which are both strictly positive, and remember that this node's interpretation is $\top$, so that both $\top$ and $\bot$ trivially model the node. Induction step: Let $v$ be an internal vtree node  and assume the statement of the theorem true for CSDD nodes normalized for $v$'s descendant. Let $n = (\{(p_i,s_i)\}_{i=1}^k, \mathbb{K}_n)$ be a decision node normalized for $v$. Let $\bo{X}$ and $\bo{Y}$ be the left respectively right variables of $v$. Now, for any instantiation $\bo{xy}$ of $\bo{XY}$:

\begin{align*}
\underline{\mathbb{P}}_n(\bo{xy}) =& \min_{\mathbb{P}(\bo{XY}) \in \mathbb{K}^n(\bo{XY})} \mathbb{P}(\bo{xy})  \\
& = \min_{ \mathbb{P}_n(\mathbf{XY}) \in \mathbb{K}^n(\mathbf{XY})} \mathop{\sum}_{i=1}^k \mathbb{P}_{p_i}(\bo{x})\cdot \mathbb{P}_{s_i}(\bo{y})\cdot \theta_i \\
& = \min_{ [\theta_1, \dots, \theta_k] \in \mathbb{K}_n(P)} \mathop{\sum}_{i=1}^k \underline{\mathbb{P}}_{p_i}(\bo{x})\cdot \underline{\mathbb{P}}_{s_i}(\bo{y})\cdot \theta_i.
 \end{align*}
 
Similarly, we can derive 
$$\up{\mathbb{P}}_n(\bo{xy}) =\max_{ [\theta_1, \dots, \theta_k] \in \mathbb{K}_n(P)} \mathop{\sum}_{i=1}^k \up{\mathbb{P}}_{p_i}(\bo{x})\cdot \up{\mathbb{P}}_{s_i}(\bo{y})\cdot \theta_i.$$
 
 We have that $\bo{xy} \models \interp{n}$ if and only if  $\bo{y} \models \interp{s_j}$ for the unique $1\leq j \leq k$ such that $\bo{x} \models \interp{p_j}$.  By induction hypothesis, this happens if and only if $\underline{\mathbb{P}}_{p_j}(\bo{x})\cdot \underline{\mathbb{P}}_{s_j}(\bo{y}) > 0$. This is equivalent to 
 $\min_{ [\theta_1, \dots, \theta_k] \in \mathbb{K}_n(P)} \mathop{\sum}_{i=1}^k \underline{\mathbb{P}}_{p_i}(\bo{x})\cdot \underline{\mathbb{P}}_{s_i}(\bo{y})\cdot \theta_i > 0$ (observe that, because $\bo{y} \models \interp{s_i}$,  $s_i \neq \bot$ and hence by definition $\theta_i $ is constrained to be strictly positive). Similarly $\bo{xy} \not\models \interp{n}$ if and only if $\bo{y} \not\models \interp{s_j}$ for the unique $1\leq j \leq k$ such that $\bo{x} \models \interp{p_j}$. By induction hypothesis, this happens if and only if $\up{\mathbb{P}}_{p_j}(\bo{x})\cdot \up{\mathbb{P}}_{s_j}(\bo{y}) = 0$. By definition of $j$ and by induction hypothesis, $\up{\mathbb{P}}_{p_i}(\bo{x}) = 0$ for all $i\neq j$, making $\max_{ [\theta_1, \dots, \theta_k] \in \mathbb{K}_n(P)} \mathop{\sum}_{i=1}^k \up{\mathbb{P}}_{p_i}(\bo{x})\cdot \up{\mathbb{P}}_{s_i}(\bo{y})\cdot \theta_i =0$.
\end{proof}

\begin{proof}[Proof of Theorem \ref{theo:evidence}]
If $n$ is a terminal node, the theorem is true by definition of Algorithm \ref{alg:lowerlik} (the computation of $\underline{\mathbb{P}}_n(\mathbf{e})$ is immediate). Let $n = (\{(p_i, s_i)\}_{i=1}^k, \mathbb{K}_n(P))$ be a decision node and assume that the theorem holds for $n$'s primes and subs. If  $l$ and $r$ are the left, respectively right sub-vtree of $v$, we have that:

\begin{align*}
\lo{\mathbb{P}}_n(\bo{e}) = & \min_{ \mathbb{P}(\mathbf{Z}) \in \mathbb{K}^n(\mathbf{Z})}\mathbb{P}(\bo{e})\\
& \stackrel{(1)}{=} \min_{ \mathbb{P}_n(\mathbf{Z}) \in \mathbb{K}^n(\mathbf{Z})} \mathop{\sum}_{i=1}^k \mathbb{P}_{p_i}(\bo{e}_l)\cdot \mathbb{P}_{s_i}(\bo{e}_r)\cdot \theta_i\\
& \stackrel{(2)}{=} \min_{ [\theta_1, \dots, \theta_k] \in \mathbb{K}_n(P)} \mathop{\sum}_{i=1}^k \min_{\mathbb{P}_{p_i}(\bo{Z}_l) \in \mathbb{K}^{p_i}}\mathbb{P}_{p_i}(\bo{e}_l)\cdot \min_{\mathbb{P}_{s_i}(\bo{Z}_r)\in \mathbb{K}^{s_i}}\mathbb{P}_{s_i}(\bo{e}_r)\cdot \theta_i\\
& \stackrel{(3)}{=} \min_{ [\theta_1, \dots, \theta_k] \in \mathbb{K}_n(P)} \mathop{\sum}_{i=1}^k \underline{\pi}(p_i)\cdot \underline{\pi}(s_i)\cdot \theta_i
\end{align*}
(1) is because optima are attained in extreme points, plus \cite[Theorem 7]{kisa2014}. In (2) we move the minimizations concerning  $\mathbb{P}_{p_i}(\bo{e}_l)$ and $ \mathbb{P}_{s_i}(\bo{e}_r)$ inside the sum. This can be done because these minimizations are done over two distinct CSs (the strong extension of the sub-CSDD rooted at $p_i$ and the strong extension of the sub-CSDD rooted at $s_i$).  and then, with the obtained values, solve the LP over the CS $\mathbb{K}_n(P)$ attached to node $n$. Hence, the induction hypothesis applies in (3), knowing again that the argument used in (1) applies to nodes $p_i$ and $s_i$, for all $1\leq i \leq k$.
\end{proof}

\begin{proof}[Proof of Theorem \ref{theo:conditional}]
Let $n$ be a node normalized for a vtree node $v$ in the input CSDD.
If $X$ does not occur in $v$, $\mathbb{P}(x,\bm{e}) = \mathbb{P}(\neg x,\bm{e}) = \mathbb{P}(\bm{e})$ for all $\mathbb{P}(\bm{X}) \in \mathbb{K}^n(\bm{X})$. The result of the right hand side minimization is then $(1-2\mu)\cdot \mathbb{P}(\bm{e})$, i.e., $\underline{\rho}_n(\mu)$.

Now assume that $X$ occurs in $v$.

If $v$ is a leaf, $n$ is a terminal node. As optimal values are attained on the borders of the domain, the left hand side of Equation \eqref{eq:gbr3} rewrites exactly as $\lambda_n(\mu)$. 
Hence, for a terminal node, the result of the right hand side minimization is $\Lambda(n)$, thus the base case is proved. Assume now that the Theorem is true for nodes normalized for $v$'s sub-vtrees.

Consider a decision node $n=(\{(p_i,s_i)\}_{i=1}^k, \mathbb{K}_n(P))$ (normalized for $v$) and assume that $X$ occurs in the left sub-vtree of $v$, $v^l$, the case when $X$ occurs in the right sub-vtree being \emph{mutatis mutandis} the same. The right hand side of the equality to be proven can be rewritten as
\small
\begin{align*}
\min_{\mathbb{P}(\bm{X}) \in \mathbb{K}^n(\bm{X})}  \left[ (1-\mu) \mathbb{P}(x,\bm{e})-\mu \mathbb{P}(\neg x,\bm{e}) \right] \stackrel{(1)}{=} \min_{\mathbb{P}_n(\bm{X}) \in \mathbb{K}^n(\bm{X})}  \left[ (1-\mu) \mathbb{P}_n(x,\bm{e})-\mu \mathbb{P}_n(\neg x,\bm{e}) \right] \\
 \stackrel{(2)}{=}  \min_{\mathbb{P}_n(\bm{X}) \in \mathbb{K}^n(\bm{X})}  \left[ (1-\mu)\sum_{i=1}^k \mathbb{P}_{p_i}(x,\bm{e}_l)\mathbb{P}_{s_i}(\bm{e}_r)\theta_i      -\mu \sum_{i=1}^k \mathbb{P}_{p_i}(\neg x, \bm{e}_l)\mathbb{P}_{s_i}(\bm{e}_r)\theta_i \right] &\\
= \min_{\mathbb{P}_n(\bm{X}) \in \mathbb{K}^n(\bm{X})}  \left[ \sum_{i=1}^k [(1-\mu)\mathbb{P}_{p_i}(x,\bm{e}_l) -\mu \mathbb{P}_{p_i}(\neg x, \bm{e}_l)]\cdot \mathbb{P}_{s_i}(\bm{e}_r)\cdot\theta_i \right] &\\
\stackrel{(3)}{=} \min_{(\theta_1,\dots , \theta_k) \in \mathbb{K}_n(P)}  \left[ \sum_{i=1}^k \min_{\mathbb{P}_{p_i}(\bm{X}_l) \in \mathbb{K}^{p_i}(\bm{X}_l)} \left[(1-\mu)\mathbb{P}_{p_i}(x,\bm{e}_l) -\mu \mathbb{P}_{p_i}(\neg x, \bm{e}_l)\right] \cdot  \min_{\mathbb{P}_{s_i}(\bm{X}_r) \in \mathbb{K}^{s_i}(\bm{X}_r)} \mathbb{P}_{s_i}(\bm{e}_r)\cdot\theta_i \right] &\\
\stackrel{(4)}{=} \min_{(\theta_1,\dots , \theta_k) \in \mathbb{K}_n(P)}  \left[ \sum_{i=1}^k \min_{\mathbb{P}(\bm{X}_l) \in \mathbb{K}^{p_i}(\bm{X}_l)} \left[(1-\mu)\mathbb{P}(x,\bm{e}_l) -\mu \mathbb{P}(\neg x, \bm{e}_l)\right] \cdot  \min_{\mathbb{P}(\bm{X}_r) \in \mathbb{K}^{s_i}(\bm{X}_r)} \mathbb{P}(\bm{e}_r)\cdot\theta_i \right] &\\
\stackrel{(5)}{=} \min_{(\theta_1,\dots , \theta_k) \in \mathbb{K}_n(P)}  \left[ \sum_{i=1}^k \underline{\pi}(p_i)\cdot \underline{\sigma}_{s_i}(\underline{\pi}(p_i))\cdot \theta_i \right] & \\
\end{align*}

\normalsize

where equalities (1) and (4) are because optimal values are attained in extreme points of the strong extension, (2) is thanks to Theorem  \cite[Theorem 6]{kisa2014}. Equality (3) is because the strong extensions of $p_i$ and $s_i$ are distinct, thus the optimization can be performed separately. Note that here the singly connectedness assumption is necessary, as explained in the last part of this section. Equality  (5) is by induction hypothesis plus $\underline{\sigma}_{s_i}$'s definition.
\end{proof}

\begin{proof}[Proof of Theorem \ref{theo:credalmap}] If $n$ is a terminal, it is easy to see the correctness of the algorithm. Suppose that $n$ is a decision node, $n = (\{(p_i,s_i)\}_{i=1}^k, \mathbb{K}_n)$. For a given $\bo{x} \in val(\bo{X})$, $\bo{xe}$ is a total instantiation of its variables. Since the maximum on $\mathbb{K}^n$ is realized on extreme points we can consider PSDDs probability distributions when computing the maximum. Remember that each considered instantiation of $\bo{X}$ selects a unique branch $1\leq i \leq k$ of $n$. With the same reasoning adopted in the proof of Algorithm \ref{alg:map}, we can argue that  
\small
\begin{align*}
\mathop{\max}_{\bo{x}\in val(\bo{X})} \mathop{\max}_{\mathbb{P} \in \mathbb{K}^n}\mathbb{P}(\bo{x},\bm{e})  & = \mathop{\max}_{1 \leq i \leq k} \mathop{\max}_{\bo{x}\in val(\bo{X})} \mathop{\max}_{\mathbb{K}_{n,i}} \theta_i \cdot  \mathop{\max}_{\mathbb{K}^{p_i}}\mathbb{P}_{p_i}(\bo{x}_l \bo{e}_l) \cdot  \mathop{\max}_{\mathbb{K}^{s_i}}\mathbb{P}_{s_i} (\bo{x}_r \bo{e}_r)\\
& = \mathop{\max}_{1 \leq i \leq k} \mathop{\max}_{\mathbb{K}_{n,i}} \theta_i \cdot [ \mathop{\max}_{\bo{x}_l\in val(\bo{X}_l)} \mathop{\max}_{\mathbb{K}^{p_i}}\mathbb{P}_{p_i}(\bo{x}_l \bo{e}_l)] \cdot  [\mathop{\max}_{\mathbb{K}^{s_i}}\mathop{\max}_{\bo{x}_r\in val(\bo{X}_r)} \mathbb{P}_{s_i}(\bo{x}_r \bo{e}_r)]\\
& = \mathop{\max}_{1 \leq i \leq k}  \mathop{\max}_{\mathbb{K}_{n,i}}\theta_i \cdot  M(p_i) \cdot  M(s_i)
\end{align*}
\normalsize
\end{proof}

\begin{proof}[Proof of Theorem \ref{theo:rob}]
Base case: Let $n$ be a terminal node. 
\begin{itemize}
\item If $var(n) \in \bo{X}$:
\begin{itemize}
\item if $n\in \{X, \neg X\}$, then if $\bo{x}_v^*\models n$ the maximization clearly reduces to $1$ , while if $\bo{x}_v^*  \not \models n$, the expression is not defined and we refer to the convention;
\item if $n= (X: [l,u])$: if $\bo{x}_v^*=\top$, 

$$\mathop{\max}_{\theta \in [l,u]} \{ \frac{\mathbb{P}_n(\top)}{\mathbb{P}_n(\top)}, \frac{\mathbb{P}_n(\bot)}{\mathbb{P}_n(\top)}\} = \max \{ \mathop{\max}_{\theta \in [l,u]} \tfrac{\theta}{\theta}, \mathop{\max}_{\theta \in [l,u]} \tfrac{1-\theta}{\theta}\} =\max \{1, \tfrac{1-l}{l}\},$$

otherwise, if $\bo{x}_v^*=\bot$, 

$$\mathop{\max}_{\theta \in [l,u]} \{ \frac{\mathbb{P}_n(\bot)}{\mathbb{P}_n(\bot)}, \frac{\mathbb{P}_n(\top)}{\mathbb{P}_n(\bot)} \}= \max \{\mathop{\max}_{\theta \in [l,u]} \tfrac{1-\theta}{1-\theta}, \mathop{\max}_{\theta \in [l,u]} \tfrac{\theta}{1-\theta} \}=\max \{1, \tfrac{u}{1-u}\},$$

\end{itemize}
\item If $var(n) \in \bo{E}$:  if $\bo{e}\models n$, the fraction reduces to $1$, while if $\bo{e}\not\models n$, again the expression is not defined hence we refer to the convention.
\end{itemize}

Induction step: let  $n = ({(p_i,s_i)}_{i=1}^k,\mathbb{K}_n(P))$ be a decision node. If $\bo{x}^*_v\bo{e}_v \not \models \interp{n}$, $V(n) = 1$, in accord with the convention. Assume now that $\bo{x}^*_v\bo{e}_v \models \interp{n}$. Since $\bo{x}^*$ is fixed, there is a unique $1\leq j \leq k$ such that $\bo{x}_{v,l}^*\bo{e}_{v,l}\models p_j$. Then ( as usual, we can perform the optimization on the extreme points of the strong extension)
$$\mathop{\max}_{\bo{x}_v\in val(\bo{X}_v)} \mathop{\max}_{\mathbb{P}_n \in \mathbb{K}^n} \frac{\mathbb{P}_n(\bo{x}_v, \bo{e}_v)}{\mathbb{P}_n(\bo{x}_v^*, \bo{e}_v)} =  \mathop{\max}_{\mathbb{P}_n \in \mathbb{K}^n} \frac{\mathop{\max}_{\bo{x}_v \in val(\bo{X}_v)}\mathbb{P}_n(\bo{x}_v, \bo{e}_v)}{\mathbb{P}_{p_j}(\bo{x}^*_{v^l}, \bo{e}_{v^l})\cdot \mathbb{P}_{s_j}(\bo{x}^*_{v^r}, \bo{e}_{v^r}) \cdot \theta_j}$$

 $$ = \mathop{\max}_{\mathbb{K}^n} \frac{\mathop{\max}_{1\leq i \leq k} \theta_i \mathop{\max}_{\bo{x}_{v^l}}\mathbb{P}_{p_i}(\bo{x}_{v^l}, \bo{e}_{v^l})\cdot \mathop{\max}_{\bo{x}_{v^r}}\mathbb{P}_{s_i}(\bo{x}_{v^r}, \bo{e}_{v^r})}{\mathbb{P}_{p_j}(\bo{x}^*_{v^l}, \bo{e}_{v^l})\cdot \mathbb{P}_{s_j}(\bo{x}^*_{v^r}, \bo{e}_{v^r}) \cdot \theta_j}$$

Now, for $1\leq i \leq k$, if $i=j$ the above expression simplifies and becomes

$$\mathop{\max}_{\mathbb{P}_{p_j} \in \mathbb{K}^{p_j}} \frac{ \mathop{\max}_{\bo{x}_{v^l}}\mathbb{P}_{p_j}(\bo{x}_{v^l}, \bo{e}_{v^l})}{\mathbb{P}_{p_j}(\bo{x}^*_{v^l}, \bo{e}_{v^l})} \cdot \mathop{\max}_{\mathbb{P}_{s_j} \in \mathbb{K}^{s_j}}\frac{\mathop{\max}_{\bo{x}_{v^r}}\mathbb{P}_{s_j}(\bo{x}_{v^r}, \bo{e}_{v^r})}{\mathbb{P}_{s_j}(\bo{x}^*_{v^r}, \bo{e}_{v^r})}$$
that is, by induction hypothesis,

$$V(p_j)\cdot V(s_j).$$

If we fix a $i \neq j$ instead, the optimizations might be performed independently since the CSs above and below are distinct:

 $$ = \mathop{\max}_{\mathbb{K}_n} \frac{\theta_i \mathop{\max}_{\bo{x}_{v^l}} \mathop{\max}_{\mathbb{P}_{p_i} \in \mathbb{K}^{p_i}}  \mathbb{P}_{p_i}(\bo{x}_{v^l}, \bo{e}_{v^l})\cdot \mathop{\max}_{\bo{x}_{v^r}}\mathop{\max}_{\mathbb{P}_{s_i} \in \mathbb{K}^{s_i}}\mathbb{P}_{s_i}(\bo{x}_{v^r}, \bo{e}_{v^r})}{\theta_j \cdot \underline{\mathbb{P}}_{p_j}(\bo{x}^*_{v^l}, \bo{e}_{v^l})\cdot \underline{\mathbb{P}}_{s_j}(\bo{x}^*_{v^r}, \bo{e}_{v^r})}$$ 
that is,

 $$ = \mathop{\max}_{\mathbb{K}_n} \frac{\theta_i \cdot M(p_j)\cdot M(s_j)}{\theta_j \cdot \underline{\mathbb{P}}_{p_j}(\bo{x}^*_{v^l}, \bo{e}_{v^l})\cdot \underline{\mathbb{P}}_{s_j}(\bo{x}^*_{v^r}, \bo{e}_{v^r})},$$
 which completes the proof.
\end{proof}

\section*{CSDD quantification for Example \ref{ex:toy_IDM}}
\begin{align*}
\theta_1 &= P(\neg X_1 \wedge \neg X_2) \in \left[ \frac{n_{\theta_1} }{n+s},\frac{n_{\theta_1} +s}{n+s} \right]\\
\theta_2 &= P((X_1\wedge \neg X_2)\vee (\neg X_1 \wedge X2))\in \left[ \frac{n_{\theta_2}}{n+s},\frac{n_{\theta_2}+s}{n+s}\right]\\
\theta_3 &= P(\neg X_3 | \neg X_1 \wedge \neg X_2 )\in \left[ \frac{n_{\theta_3}}{n_{\theta_1}+s},\frac{n_{\theta_3}+s}{n_{\theta_1}+s}\right]\\
\theta_4 &= P(X_1|(X_1 \wedge \neg X_2)\vee (\neg X_1 \wedge X_2))\in \left[ \frac{n_{\theta_4}}{n_{\theta_2}+s},\frac{n_{\theta_4}+s}{n_{\theta_2}+s}\right]\\
\theta_5 &= P(X_3| (X_1 \wedge \neg X_2)\vee (\neg X_1 \wedge X_2))\in \left[ \frac{n_{\theta_5}}{n_{\theta_2}+s},\frac{n_{\theta_5}+s}{n_{\theta_2}+s}\right]\\
\theta_6 &= P(X_4|(\neg X_1 \wedge \neg X_2)\wedge X_3)\in \left[ \frac{n_{\theta_6}}{n_{\theta_1} -n_{\theta_3} +s},\frac{n_{\theta_6}+s}{n_{\theta_1} -n_{\theta_3}+s}\right]\\
\theta_7 &= P(X_4|(X_1 \wedge \neg X_2)\vee (\neg X_1 \wedge X_2)\wedge \neg X_3)\in \left[ \frac{n_{\theta_7}}{n_{\theta_2} -n_{\theta_5}+s},\frac{n_{\theta_7}+s}{n_{\theta_2} -n_{\theta_5}+s}\right]
\end{align*}

 where

\begin{align*}
n_{\theta_1} & =  n_2+n_6+n_9\\
n_{\theta_2} & =  n_0+n_1+n_4+n_5+n_7+n_8\\
n_{\theta_3} & =  n_6\\
n_{\theta_4} & =  n_0+n_5+n_8\\
n_{\theta_5} & = n_1+n_5 \\
n_{\theta_6} & = n_2 \\
n_{\theta_7} & =  n_0+n_4
\end{align*}

\bibliography{biblio}
\end{document}